\documentclass[11pt]{article}
\usepackage{float}

\usepackage[final]{acl}

\usepackage{times}
\usepackage{latexsym}
\usepackage[T1]{fontenc}
\usepackage[utf8]{inputenc}
\usepackage{microtype}
\usepackage{inconsolata}
\usepackage{graphicx}
\usepackage{cite}
\usepackage{subcaption}
\usepackage{amsmath,amssymb,amsfonts}
\usepackage{algorithm}
\usepackage{algorithmic}
\usepackage{textcomp}
\usepackage{xcolor}
\usepackage{multirow}
\usepackage{booktabs}
\usepackage{stfloats}
\usepackage{comment}
\newtheorem{proof}{Proof}
\newtheorem{theorem}{Theorem}
\newtheorem{proposition}{Proposition}
\usepackage{tabularx}

\title{Self-Reinforcing Controllable Synthesis of Rare Relational Data via Bayesian Calibration}

\author{
  Chongsheng Zhang\thanks{\ \ Equal contribution}\thanks{\ \ Corresponding author}$^{1,2}$~~~
  Hao Wang\footnotemark[1]$^{1}$~~~
  Zelong Yu\footnotemark[1]$^{1}$~~~
  Esteban Garces Arias\footnotemark[1]$^{2,3}$\\
  \textbf{Julian Rodemann}\footnotemark[1]$^{2,4}$~~~
  \textbf{Zhanshuo Zhang}$^{1}$~~~
 \textbf{Qilong Li}$^{1}$~~~
  \textbf{Gaojuan Fan}$^{1}$~~~\\
    \textbf{Krikamol Muandet}$^{4}$~~~
  \textbf{Christian Heumann}$^{2}$\\
\\
  $^1$Henan University, China $^2$Department of Statistics, LMU Munich
  $^3$MCML Munich\\
  $^4$CISPA Helmholtz Center for Information Security, Saarbrücken, Germany\\
   \small{\textbf{Correspondence:} \href{mailto:cszhang@henu.edu.cn}{cszhang@henu.edu.cn}}
}

\begin{document}
\maketitle
\begin{abstract}
Imbalanced data are commonly present in real-world applications. While data synthesis can effectively mitigate data scarcity for rare classes, and LLMs have revolutionized text generation, the application of LLMs to the synthesis of relational/structured tabular data remains underexplored. Moreover, existing approaches lack an effective feedback mechanism to guide LLMs in continuously optimizing the quality of the generated data throughout the synthesis process. In this work, we propose RDDG, \textbf{R}elational \textbf{D}ata generator with \textbf{D}ynamic \textbf{G}uidance, which is a unified in-context learning framework that employs progressive chain-of-thought (CoT) steps to generate tabular data for enhancing downstream imbalanced classification performance. RDDG first uses core set selection to identify representative samples from the original data, then utilizes in-context learning to discover the inherent patterns and correlations among attributes within the core set, and subsequently generates tabular data while preserving the aforementioned constraints. More importantly, it incorporates a self-reinforcing feedback mechanism that provides automatic assessments of the quality of the generated data, enabling continuous quality optimization throughout the generation process. Experimental results on multiple real and synthetic datasets demonstrate that RDDG outperforms existing approaches in both data fidelity and downstream imbalanced classification performance. We make our code available at \url{https://github.com/cszhangLMU/RDDG}.
\end{abstract}

\section{Introduction}
The scarcity of data, particularly annotated data, has been a pervasive challenge in deep learning and the new era of artificial intelligence (AI). Automated data synthesis techniques are now widely used to address data scarcity, particularly with the rise of large foundation models. For instance, large language models (LLMs) have been widely used in various text generation and analysis tasks \citep{llmsurvey2023,llmsurvey2024}, while multi-modal large foundation models have also been used to generate image data to enhance visual learning performance \citep{ltgc}. 

Despite the revolutionary impact of foundation models and their widespread applications in \textit{unstructured} text and image generation, their potential for relational or structured tabular data synthesis is still underexplored \citep{epic}, whereas non-LLM-based approaches have already demonstrated their effectiveness in relational data generation for enhancing imbalanced classification performance \citep{ag,tvae,tabddpm}. Moreover, there is no internal self-reinforcing feedback mechanism to guide foundation models to continuously optimize the quality of generated data throughout the in-context data synthesis process. 

To address the above challenges, we propose RDDG, a novel framework with progressive chain-of-thought (CoT) steps and dynamic guidance for in-context synthesis of tabular data using LLMs. RDDG first performs core set selection based on sample error variance to identify the most representative training samples, which are then fed into LLMs along with initial prompts that contain dataset and attribute information, thereby guiding the LLMs to analyze functional relationships among attributes. With these patterns and constraints in hand, RDDG continually feeds each batch of real data to LLMs, guiding them to generate semantically meaningful samples. This procedure incorporates a self-reinforcing feedback mechanism that automatically assesses the quality of the generated data relative to real data. The evaluation results are then converted into feedback prompts to be integrated into the subsequent in-context learning, enabling continuous optimization of data quality throughout the synthesis process.

\paragraph{Contributions} The main contributions of this paper are as follows:
\begin{itemize}
\item We propose a unified in-context learning framework with progressive CoT steps for generating tabular data to enhance imbalanced classification, which obtains domain-specific prior knowledge by mining functional relationships, leading to constraint-driven, controllable synthesis of tabular data.

\item We devise a self-reinforcing feedback mechanism that uses multiple quality measures to automatically assess the quality of the current batch of generated data, which is then converted into feedback prompts and propagated to the following in-context learning process, thereby continuously improving the realism of the generated data throughout the synthesis process.

\item We formulate the self-reinforcing feedback process as a Bayesian calibration problem, prove Bayes-optimality of our framework (Theorem~\ref{thm:bayes-opt}), and show that our feedback mechanism targets these optimal strategies under some assumptions (Proposition~\ref{prop}).

\item We conduct extensive experiments on eight datasets. Compared to the state-of-the-art in-context learning approaches for imbalanced classification, RDDG achieves average improvements of more than 2\% and 1\% in the weighted Macro-F1 metric and Balanced Accuracy, respectively, while simultaneously preserving superior data fidelity.
\end{itemize}

\section{Related Work}

Based on neural architectures, deep data generation methods can be divided into statistical distribution-based methods \citep{leap,ladc}, GAN-based methods \citep{ctgan,tablegan,glgan,ref10,ctabganplus}, diffusion model-based methods \citep{tabddpm,CDTD}, and foundation model-based methods \citep{whytabllm24,tabllmnature2025,TabPFN,great,epic}. 

Statistical distribution-based methods focus on modeling underlying data distributions and generate samples via distribution sampling, whereas GAN-based methods utilize the GAN architecture for data generation. TableGAN \citep{tablegan} pioneered GAN-based tabular data synthesis by incorporating an additional classifier module to co-supervise the generation process. GLGAN \citep{glgan} first adopts SMOTE \citep{smote} to create minority samples, then utilizes GANs to learn underlying data distributions. CTAB-GAN \citep{ref10} and CTAB-GAN+ \citep{ctabganplus} jointly optimize adversarial and classification losses to improve minority class sample generation capabilities.

TabDDPM \citep{tabddpm} investigates Diffusion model-based tabular data synthesis, generating high-quality samples through forward noise injection and reverse denoising processes. CDTD \citep{CDTD} adapts the diffusion model for generating mixed-type features, adaptively scaling the respective losses for numeric and categorical features.

The emergence of pretrained foundation models has opened new avenues for data synthesis \citep{whytabllm24}, and existing approaches in this direction can be further categorized into fine-tuning and prompt-based methods. TabPFN \citep{tabllmnature2025,TabPFN} is a Transformer-based tabular foundation model that modifies the self-attention mechanism so that training samples can only attend to other training samples, pre-trained on 100 million synthetic datasets generated using structural causal models. Unlike TabPFN, TABULA-8B \citep{tabula8b} builds a tabular foundation model by fine-tuning Llama 3-8B using a filtered 4-million high-quality subset of real-world web-crawled tables, converting relational records into template-based prompts. Similarly, LLM-GTL \citep{GTL} converts tabular samples into template-based, instruction-oriented prompts and fine-tunes Llama-2 using 340 real-world Kaggle datasets. GReaT \citep{great} and EPIC \citep{epic} employ prompt-based methods that leverage in-context learning to guide LLMs in generating synthetic tabular data without requiring model updates or fine-tuning.

\noindent \textbf{Discussion}. Despite the above advances in LLM-based tabular data generation, i) a significant gap exists between data generation methods and downstream task optimization goals, particularly imbalanced classification; and ii) there is a lack of an internal feedback mechanism that can guide LLMs to continuously optimize the quality of the generated data throughout the in-context learning process.

\section{Methodology}

Given a sample of real data $\mathcal{R}$ drawn from some unknown distribution $\mathbb{P}_{\mathcal{R}}$, our goal is to generate synthetic data $\mathcal{S}$ via $ \mathcal{S} = S_\phi(\mathcal{R}, \mathcal{C}, \mathcal{F})$ where $S_\phi(\cdot, \cdot, \cdot)$ is the generator parametrized by $\phi$, $\mathcal{C}$ is a set of constraints, and $\mathcal{F}$ is the feedback. The feedback $\mathcal{F}$ constitutes the \textit{self-reinforcing} synthesis mechanism, while the constraints $\mathcal{C}$ render the method \textit{controllable}.

Our procedure is sequential, i.e., for $i \in \mathcal{I}$, batches of synthetic data $\mathcal{S}_i$ are generated by
\begin{equation}
\mathcal{S}_i = S_\phi(\mathcal{R}_i, \mathcal{C}, \mathcal{F}_{i-1})
\end{equation}
with $\mathcal{I}$ being an index set, typically a finite subset of natural numbers $\mathbb{N}$. Notably, the sequential setup invokes the Bayesian perspective on simulation \citep{wade2022bayesian}, as detailed in Section \ref{sec:bayesian-view}.

\noindent \textbf{Objective } At any iteration $i \in \mathcal{I}$, we aim to generate synthetic data $\mathcal{S}_i$ that is close to the unknown distribution $\mathbb{P}_{\mathcal{R}}$ with respect to some divergence measure $d$ (e.g., Kullback-Leibler (KL) divergence), i.e., 
\begin{equation}
    \min_{\mathcal{S}_i}\; d(\hat{ \mathbb{P}}_{\mathcal{S}_i},
    \mathbb{P}_{\mathcal{R}}) 
\end{equation}
with $\hat{ \mathbb{P}}_{\mathcal{S}_i}$ being the empirical measure of $\mathcal{S}_i$
for any $i \in \mathcal I$. In the following, we present our method that aims at achieving this goal by approximating $\arg\min_{\mathcal{S}_i} 
    d(\hat{ \mathbb{P}}_{\mathcal{S}_i},
    \mathbb{P}_{\mathcal{R}})$, including theoretical intuition based on a Bayesian perspective in Section~\ref{sec:bayesian-view}.

\subsection{Overall Framework}
In Fig. \ref{fig:framework}, we present our RDDG framework, i.e., \underline{\textbf{R}}elational \underline{\textbf{D}}ata generator with \underline{\textbf{D}}ynamic \underline{\textbf{G}}uidance, which is a progressive in-context learning framework with an internal self-reinforcing feedback mechanism to produce constraint-compliant, high-fidelity relational data for enhancing downstream imbalanced classification performance.

\begin{figure*}[htbp]
    \centering
    \includegraphics[width=1\linewidth]{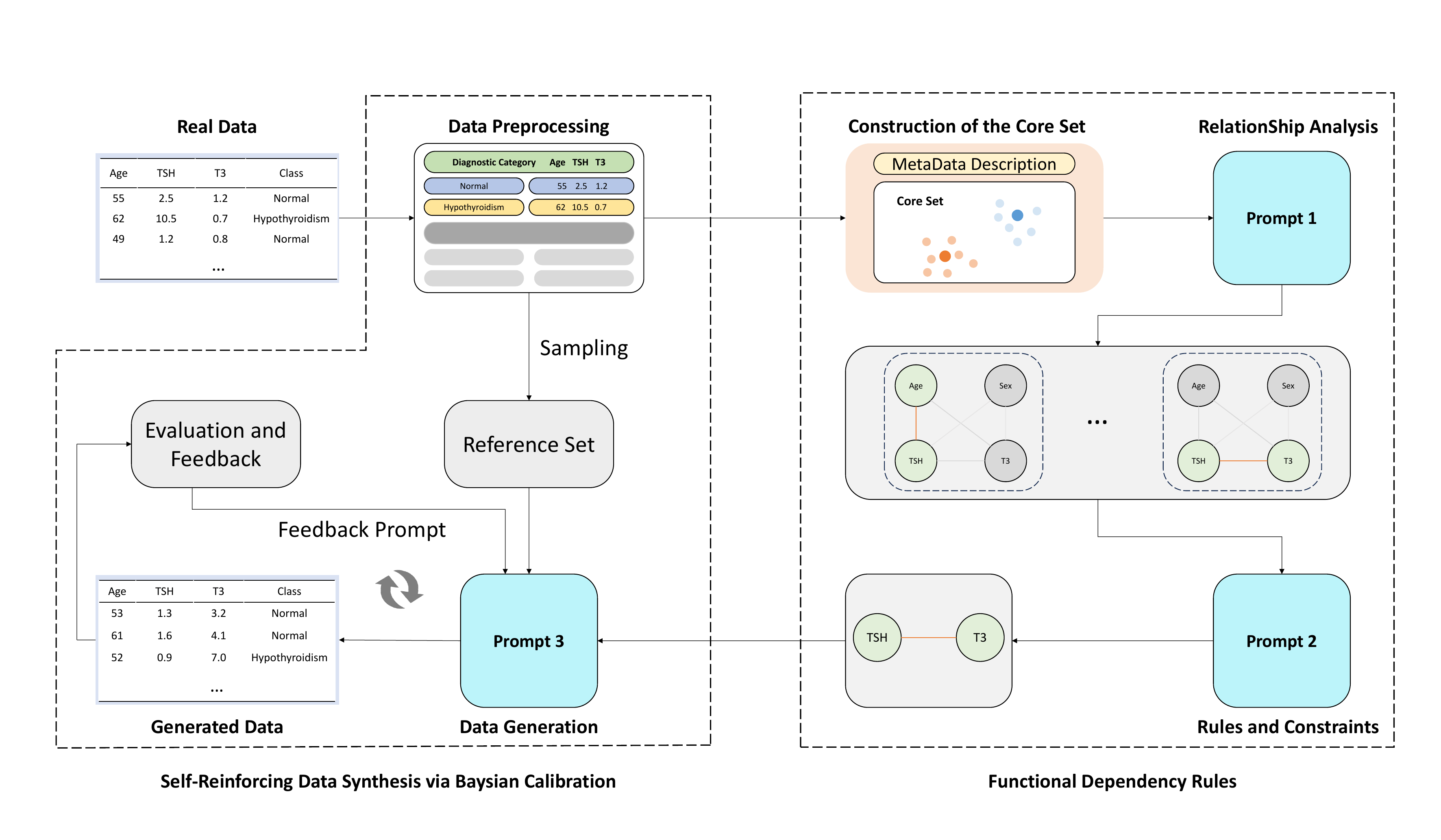}
    \vspace{-20pt}
    \caption{Overall framework of RDDG, consisting of three main steps, which are core set construction, relation mining, and data generation and constraint optimization. (i) Core set construction selects representative samples from original data to address the context-length limitation of LLMs; (ii) Relation mining uncovers the latent functional relationships (patterns and inter-attribute correlations) from core set (prompts 1 and 2); (iii) Batch-wise data generation considers both the functional relationships from step (ii) and a batch of reference set (prompt 3\_1), and devises a self-reinforcing feedback mechanism (prompt 3\_2) for continuously optimizing in-context learning.} \label{fig:framework}
\end{figure*}

Prior to in-context learning, given the context-window limitation of LLMs, we first curate a representative core set from the training data to ensure comprehensive coverage of the data's characteristics. 

The second step is relation mining, in which we use in-context learning to uncover latent patterns and inter-attribute relationships within the core set, which are then established as explicit structural constraints to guide LLMs in the generation process.

The third step is data generation and constraint optimization, in which RDDG incorporates an internal self-reinforcing feedback mechanism that automatically assesses the quality of the generated data from the preceding batch, with the corresponding evaluation results being converted into feedback prompts and incorporated with the explicit constraints into subsequent in-context learning, to guide LLMs to continuously improve the quality/fidelity of the generated data. This batch-wise generation paradigm continues until the total number of synthetic samples reaches the target threshold. A detailed overview of these steps is provided as a pseudo-algorithm in Appendix \ref{sec:pseudoalgo}.

\subsection{Core Set Construction}

To provide high-quality training samples to LLMs under prompt-length constraints, we propose grafting the core set sample selection method \citep{coreset}, originally designed for processing high-dimensional medical data in resource-limited environments, to address the context-length limitation of LLMs.

In our design, we use a straightforward MLP model and split the overall training process into early, middle, and late stages. As with \citet{coreset}, we aim to identify samples that display high variance in prediction error across both early and late training stages. Specifically, we calculate the prediction error (L2 error) of the MLP model on each tabular sample at every epoch, denoted in Equation \ref{eq:l2_loss}, then calculate the mean and variance of each sample's prediction errors across both the early and late training stages/epochs. Finally, for each class, we select the Top-K samples that exhibit the highest variance (\textit{Var}) in prediction errors, as depicted in Equation \ref{k_set}. If the number of samples in a class is less than K, we repeat sampling for that class. 

\begin{equation}
\mathcal{L}_2(\mathbf{y}_{\text{pred}}, \mathbf{y}_{\text{true}}) = \|\mathbf{y}_{\text{pred}} - \mathbf{y}_{\text{true}}\|_2^2
\label{eq:l2_loss}
\end{equation}

\begin{equation}
    \text{Top}_{k}(k) = \arg top_{K}([\text{Var}_{i} \mid i \in N_{k}])
    \label{k_set}
\end{equation}

Overall, our class-balanced core set selection strategy identifies, for each class, the Top-K samples with the highest variance in prediction errors, ensuring that each class receives fair attention when fed into LLMs. More details are given in Appendix \ref{Appendix-data}.

\subsection{Progressive In-Context Data Synthesis}

\noindent \textbf{Relation Mining}. Our RDDG framework resembles the ``chain-of-thought (CoT)'' process \citep{cot}, with step-by-step reasoning steps to guide LLMs in the generation process. To generate realistic synthetic data, we guide LLMs to first analyze the functional relationships (mainly the patterns and inter-attribute relationships) from the core set (prompt 1); next, we instruct LLMs to comprehensively take the core set, metadata, and the above relationships into account to establish explicit rules and constraints for data generation (prompt 2).

\noindent \textbf{Data Generation}. In the data generation phase, the original training set is partitioned into multiple subsets (batches) and used as reference sets for in-context learning. In each iteration, given a reference set and the above rules and constraints from relation mining, we guide LLMs to generate synthetic samples (prompt 3\_1). 

To further control the quality of the generated data, we introduce the self-reinforcing feedback mechanism for constraint optimization below. 

\subsection{Dynamic Guidance Adjustment}\label{sec:guidance-adj}

To ensure continuous improvement in the quality of generated data across subsequent batches, we integrate a novel self-reinforcing feedback mechanism into the synthesis process, which provides automatic evaluations of the current batch's quality to dynamically guide in-context learning. 

Specifically, for each batch $i$ of real data (reference set), we generate a corresponding batch of synthetic data, then immediately evaluate its quality through three key dimensions: Statistical Consistency, which involves comparing means and standard deviations between generated and real data; Correlation Preservation, utilizing Pearson correlation coefficients to assess the maintenance of inter-attribute relationships; and Distribution Consistency, employing Kolmogorov-Smirnov tests to verify distributional alignment.

Crucially, the feedback from evaluating batch~$i$ is not used to regenerate the same batch, but rather incorporated as additional guidance when processing batch $i+1$ with a new subset of real data. Formally, let $\mathcal{R}_i$ denote the $i$-th batch of real data (reference set) and $\mathcal{S}_i$ the corresponding synthetic batch. The generation process follows:

\begin{equation}
\mathcal{S}_i = S_\phi(\mathcal{R}_i, \mathcal{C}, \mathcal{F}_{i-1})
\end{equation}

\noindent where $S_\phi$ is the generation mechanism (treated as a random variable) with hyperparameter $\phi$, $\mathcal{C}$ represents the constraints derived from relation mining, and $\mathcal{F}_{i-1}$ represents the feedback guidance computed from evaluating the previous batch $\mathcal{S}_{i-1}$ against $\mathcal{R}_{i-1}$.

This forward-propagating feedback mechanism creates a self-optimizing generation pipeline where each iteration benefits from insights gained in previous rounds/batches. The quality evaluation results are then transformed into a prompt (prompt 3\_2) and incorporated into the existing prompt (prompt 3\_1) for the next batch generation. This sequential refinement continues until the total number of synthetic samples reaches the target threshold, progressively enhancing the semantic consistency and statistical fidelity of synthesized data. Detailed prompt designs are provided in Tables \ref{tab:prompt_design}, \ref{fullprompt1}, and \ref{fullprompt2} in the Appendix \ref{Appendix-prompt}.

\subsection{A Bayesian View on Dynamic Guidance Adjustment}
\label{sec:bayesian-view}

Our RDDG framework with a self-reinforcing feedback mechanism can be understood from a Bayesian perspective on simulation \citep{wade2022bayesian,poole2000inference,andradottir2000applying}.
Recall from Section~\ref{sec:guidance-adj} that we generate synthetic batches $\mathcal{S}_i$ via $ \mathcal{S}_i = S_\phi(\mathcal{R}_i, \mathcal{C}, \mathcal{F}_{i-1})$ with hyperparameters $\phi$, where $\mathcal{R}_i$ is real data, $\mathcal{C}$ are constraints, and $ \mathcal{F}_{i-1}$ being the feedback guidance from previous $\mathcal{S}_{i-1}$ against $\mathcal{R}_{i-1}$. Bayesian calibration \citep{wade2022bayesian} treats $\phi$ as unknown and places a prior $p(\phi)$ encoding structural beliefs discovered during the relation mining phase.

Given $\mathcal{R}_i$, $S_\phi$ and summary targets $T(\mathcal{R})$ (means, standard deviations, Pearson correlations, KS distances), we posit 
a likelihood $p\!\left(T(\mathcal{R})\mid T\!\left(S_\phi\right)\right)$ that scores synthetic batches against $\mathcal{R}$. Calibration then infers the posterior $p(\phi \mid T(\mathcal{R})) \propto p\!\left(T(\mathcal{R})\mid T\!\left(S_\phi\right)\right)p(\phi)$. The \emph{closed loop} appears as sequential Bayesian updating over batches $i=1,2,\dots$, where feedback metrics $F_i$ act as posterior predictive checks; each update nudges $\phi$ toward regions that simultaneously preserve functional relations and improve class-imbalance targets, thereby shrinking the discrepancy while quantifying uncertainty in $\phi$ and induced predictions. This mirrors standard Bayesian calibration steps \citep{wade2022bayesian} and enables casting Bayes-optimal prompting as posterior expected-utility maximization.

\begin{theorem}[Bayes-optimal prompting]\label{thm:bayes-opt}
Let $\phi\in\Phi$ (compact) denote the prompt/control of the synthesizer 
$S_\phi$, $T(\mathcal{R})$ be the target summaries from the real data, 
and the Bayesian calibration posterior for $\phi$ be
\[
\pi(\phi\mid T(\mathcal{R})) \;\propto\; 
p\!\left(T(\mathcal{R})\mid T(S_\phi)\right)\,p(\phi).
\]
Fix a bounded, jointly upper semicontinuous utility 
$U(\phi, \phi';\, T(\mathcal{R}))$ in $(\phi,\phi') \in \Phi \times \Phi$,
measuring the performance of action $\phi$ when the true synthesizer 
parameter is $\phi'$.\footnote{For instance, 
$U(\phi, \phi';\,T(\mathcal{R}))$ could trade off fidelity to 
$T(\mathcal{R})$ and task utility, e.g.,
$
U(\phi,\phi';\,T(\mathcal{R})) \;=\;
  -\alpha\,\Delta\!\bigl(T(\mathcal{R}),\,T(S_\phi)\bigr)
  \;+\;\beta\,\mathcal{U}_{\mathrm{task}}(S_{\phi'}),
$
with $\alpha,\beta\geq 0$ and $\Delta,\mathcal{U}_{\mathrm{task}}$ 
measurable, so that the expectation over $\phi'\sim\pi$ non-trivially 
integrates the task utility over posterior uncertainty.}
Then the Bayes-optimal prompt
\[
\phi^\star \;\in\; 
\arg\max_{\phi\in\Phi}\; 
\mathbb{E}_{\phi'\sim \pi(\cdot\mid T(\mathcal{R}))}\!
\big[\,U(\phi,\,\phi';\,T(\mathcal{R}))\,\big]
\]
exists (by compactness of $\Phi$ and upper semicontinuity of $U$) and 
minimizes posterior expected regret against every admissible 
\small
$\tilde\phi\in\Phi$:
\[
\mathbb{E}_{\phi'\sim \pi(\cdot\mid T(\mathcal{R}))}\!
\big[\,
  U(\tilde\phi,\,\phi';\,T(\mathcal{R})) 
  - U(\phi^\star,\,\phi';\,T(\mathcal{R}))
\,\big]\;\le\;0.
\]
\end{theorem}

\begin{proof}[Sketch]
Define the posterior expected utility functional
$\mathcal{V}(\phi) := 
\mathbb{E}_{\phi'\sim\pi(\cdot\mid T(\mathcal{R}))}
[U(\phi,\phi';\,T(\mathcal{R}))]$.
Since $U$ is bounded and jointly upper semicontinuous, 
$\mathcal{V}(\phi)$ is upper semicontinuous in $\phi$.  
Compactness of $\Phi$ then guarantees existence of 
$\phi^\star \in \arg\max_{\phi\in\Phi}\mathcal{V}(\phi)$ 
by the extreme value theorem.  
By standard Bayesian decision theory 
\citep{berger2013statistical}, maximising $\mathcal{V}$ 
is precisely the Bayes rule: for every $\tilde\phi\in\Phi$,
\[
\mathcal{V}(\phi^\star) \;\ge\; \mathcal{V}(\tilde\phi),
\]
which is equivalent to the stated regret inequality.  
Under strict concavity of $\mathcal{V}$ and convexity of $\Phi$, 
$\phi^\star$ is unique. $\blacksquare$
\end{proof}

This standard result from Bayesian decision theory implies that the self-reinforcing feedback mechanism can target the Bayes-optimal prompt, which is central to the following proposition.

\begin{proposition}\label{prop}
Let $\mathcal{V}(\phi) := 
\mathbb{E}_{\phi'\sim \pi(\cdot\mid T(\mathcal{R}))}
\bigl[U(\phi,\,\phi';\,T(\mathcal{R}))\bigr]$
denote the posterior expected utility, with $U$ as in 
Theorem~\ref{thm:bayes-opt} depending on \emph{both} the 
action $\phi$ and the uncertain parameter $\phi'$.
Suppose the self-reinforcing synthesis loop generates batches 
$\mathcal{S}_i$ yielding a stochastic supergradient $g_i$ 
satisfying:
\begin{itemize}
  \item \textbf{(Unbiasedness)} 
    $\mathbb{E}\bigl[g_i \mid \mathcal{F}_{i-1}\bigr] 
     \;\in\; \partial^+_\phi\,\mathcal{V}(\phi_{i-1})$,
    where $\partial^+_\phi$ denotes the superdifferential 
    with respect to $\phi$;
  \item \textbf{(Bounded variance)} 
    $\mathbb{E}\bigl[\|g_i\|^2 \mid \mathcal{F}_{i-1}\bigr] 
     \;\leq\; C$ 
    for some constant $C < \infty$.
\end{itemize}
Let the update be 
$\phi_i = \Pi_\Phi\{\phi_{i-1} + \eta_i\, g_i\}$,
where $\Pi_\Phi$ is the projection onto a compact \emph{convex} 
set $\Phi$, with non-increasing step-sizes $\eta_i > 0$ 
satisfying $\sum_i \eta_i = \infty$ and 
$\sum_i \eta_i^2 < \infty$. 
Then $\phi_i$ converges almost surely to the set of stationary 
points of $\mathcal{V}$ on $\Phi$, i.e.,
\[
  \mathrm{dist}\!\left(\phi_i,\, 
  \Phi_{\mathrm{stat}}\right) \;\to\; 0 
  \quad \text{a.s.},
\]
where 
$\Phi_{\mathrm{stat}} := 
\{\phi \in \Phi : 
 0 \in \partial^+_\phi\,\mathcal{V}(\phi) + 
 \mathcal{N}_\Phi(\phi)\}$
and $\mathcal{N}_\Phi(\phi)$ is the normal cone of $\Phi$ 
at $\phi$.
If, moreover, $\mathcal{V}$ is strictly concave on $\Phi$, 
then $\Phi^\star = \{\phi^\star\}$ is the unique stationary 
point and
\[
  \mathrm{dist}(\phi_i, \Phi^\star) 
  \;=\; \|\phi_i - \phi^\star\| 
  \;\to\; 0 \quad \text{a.s.}
\]
\end{proposition}

\begin{proof}[Sketch]
Define the posterior expected utility
$\mathcal{V}(\phi) :=
\mathbb{E}_{\phi' \sim \pi(\cdot \mid T(\mathcal{R}))}
[U(\phi,\phi';\,T(\mathcal{R}))]$.
Under assumptions (Unbiasedness) and (Bounded variance) 
together with the Robbins--Monro step-size conditions 
\citep{robbins1951stochastic}, the projected stochastic 
supergradient ascent iterates $\phi_i = 
\Pi_\Phi\{\phi_{i-1} + \eta_i g_i\}$ satisfy the 
conditions of the projected stochastic approximation 
theorem. This guarantees 
almost sure convergence to the set of stationary points 
$\Phi_{\mathrm{stat}}$ of $\mathcal{V}$ on the compact 
convex set $\Phi$. Under strict concavity of $\mathcal{V}$, the set 
$\Phi_{\mathrm{stat}}$ reduces to the unique global 
maximiser $\phi^\star$, so 
$\phi_i \to \phi^\star$ almost surely. In our setting, $g_i$ is constructed from the feedback 
metrics $\mathcal{F}_i$ (statistical consistency, 
correlation preservation, and KS-test distances) using 
differentiable surrogates for $\Delta$ and 
$\mathcal{U}_{\mathrm{task}}$ (e.g., smooth divergence 
approximations and differentiable proxy metrics). 
The (Unbiasedness) and (Bounded variance) assumptions hold 
provided these surrogates are unbiased estimators of 
$\partial^+_\phi \mathcal{V}$ with uniformly bounded 
second moments --- conditions that must be verified for 
any specific choice of surrogate. $\blacksquare$
\end{proof}

\section{Experiments}
\subsection{Experimental Settings} \label{expset}

\textbf{Datasets}. We use four real-world classification datasets from diverse domains, which are \textit{Travel, Sick, Heloc,} and \textit{Thyroid}, and four synthetic datasets with explicit inter-attribute correlations, which are \textit{Consumer Behavior, Health Metrics, Real Estate}, and \textit{Social Network}. Details of the datasets are provided in Appendix \ref{Appendix-data}. As with EPIC \citep{epic}, each dataset is randomly split into 80\% training and 20\% test sets.

\noindent \textbf{Comparative Methods}. We compare RDDG against representative generative methods for tabular data synthesis, including GReaT \citep{great}, EPIC \citep{epic}, TabDDPM \citep{tabddpm}, CDTD \citep{CDTD}, REalTabFormer \citep{REalTabFormer}, and ADS-GAN \citep{ag}. We also report the vanilla classification performance without using any synthetic data, denoted as ``Original''.

\noindent \textbf{Foundation Models}. In our experiments, GPT-3.5 (GPT-3.5-turbo-0125) is used as our default LLM. Additionally, on real datasets, we will report the performance of RDDG alongside other representative LLMs, i.e., Llama 3.0 and Mistral Max.

\noindent \textbf{Configurations}. For core set selection, details about MLP implementation and training configurations are provided in Table \ref{coresetconfig} in Appendix \ref{Appendix-data}.  

\subsection{Experimental Results}
\subsubsection{Imbalanced Classification Results}

\begin{table*}[htb]
\centering
\begin{minipage}{0.49\textwidth}
\centering
\resizebox{\textwidth}{!}{
\begin{tabular}{llcccc}
\toprule
\textbf{Dataset} & \textbf{Method} & \textbf{Macro-F1} & \textbf{BAL ACC} & \textbf{Sensitivity} & \textbf{Specificity} \\
\midrule
\multirow{7}{*}{Travel}
 & Original        & 58.12$\pm$2.04 & 71.21$\pm$1.56 & 56.48$\pm$2.81 & 85.63$\pm$0.85 \\
 & ADS-GAN         & 56.07$\pm$8.25 & 68.94$\pm$5.12 & 53.16$\pm$3.46 & \textbf{88.25$\pm$3.31} \\
 & REalTabFormer   & 53.25$\pm$4.10 & 67.70$\pm$2.69 & 58.42$\pm$4.19 & 86.82$\pm$1.31 \\
 & GReaT           & 60.95$\pm$2.59 & 72.86$\pm$1.80 & 58.80$\pm$3.69 & \underline{86.92$\pm$0.79} \\
 & TabDDPM & 65.32$\pm$1.95 & 73.19$\pm$2.15 & 71.98$\pm$2.89 & 84.14$\pm$1.56 \\
& CDTD      & 66.32$\pm$2.13 & 74.82$\pm$0.91 & 72.06$\pm$2.19 & 85.23$\pm$3.12 \\
 & EPIC            & \underline{66.65$\pm$2.53} & \underline{78.23$\pm$2.10} & \underline{78.00$\pm$4.59} & 78.46$\pm$2.50 \\
 & \textbf{RDDG (Ours)} & \textbf{68.63$\pm$2.12} & \textbf{79.67$\pm$4.68} & \textbf{78.23$\pm$1.23} & 82.67$\pm$2.56 \\
\midrule
\multirow{7}{*}{Sick}
 & Original        & 87.82$\pm$2.46 & 91.22$\pm$0.93 & 82.84$\pm$1.76 & 99.61$\pm$0.28 \\
 & ADS-GAN         & 87.52$\pm$1.64 & 89.82$\pm$0.53 & 79.86$\pm$0.97 & \textbf{99.82$\pm$0.16} \\
 & REalTabFormer   & 85.17$\pm$2.04 & 89.30$\pm$1.16 & 79.02$\pm$2.28 & 99.57$\pm$0.12 \\
 & GReaT           & 87.23$\pm$1.86 & 90.83$\pm$1.09 & 82.06$\pm$2.10 & \underline{99.60$\pm$0.09} \\
 & TabDDPM & \underline{88.16$\pm$2.79} & 91.89$\pm$1.52 & 84.24$\pm$2.90 & 99.55$\pm$0.17 \\
& CDTD      & \textbf{89.63$\pm$1.71} & \underline{93.25$\pm$0.97} & \underline{86.96$\pm$1.87} & 99.53$\pm$0.12 \\
 & EPIC            & 85.08$\pm$1.89 & 92.45$\pm$0.68 & 85.98$\pm$1.31 & 98.93$\pm$0.27 \\
 & \textbf{RDDG (Ours)} & 87.99$\pm$0.91 & \textbf{93.62$\pm$0.95} & \textbf{88.04$\pm$1.93} & 99.20$\pm$0.07 \\
\midrule
\multirow{7}{*}{Heloc}
 & Original        & \textbf{71.01$\pm$0.47} & 72.21$\pm$0.37 & 67.19$\pm$0.86 & 78.52$\pm$0.38 \\
 & ADS-GAN         & 70.00$\pm$0.29 & 71.08$\pm$0.24 & 67.68$\pm$0.28 & 78.52$\pm$0.62 \\
 & REalTabFormer   & 70.09$\pm$0.19 & 72.28$\pm$0.86 & 67.54$\pm$0.35 & \underline{78.82$\pm$0.46} \\
 & GReaT           & 70.25$\pm$0.33 & 72.16$\pm$0.24 & 66.22$\pm$0.49 & \textbf{79.70$\pm$0.21} \\
 & TabDDPM & \underline{70.33$\pm$0.22} & 72.49$\pm$0.19 & \textbf{67.89$\pm$0.39} & 77.09$\pm$0.40 \\
& CDTD      & 70.18$\pm$0.48 & 72.40$\pm$0.35 & 67.59$\pm$0.79 & 77.21$\pm$0.32 \\
 & EPIC            & 70.08$\pm$0.51 & \underline{72.52$\pm$0.38} & 66.90$\pm$0.75 & 78.13$\pm$0.17 \\
 & \textbf{RDDG (Ours)} & 70.32$\pm$0.55 & \textbf{72.54$\pm$0.43} & \underline{67.72$\pm$0.78} & 77.35$\pm$0.13 \\
\midrule
\multirow{7}{*}{Thyroid}
 & Original        & 94.23$\pm$1.99 & 95.08$\pm$1.60 & 91.14$\pm$3.12 & 99.02$\pm$1.01 \\
 & ADS-GAN         & 76.40$\pm$6.58 & 81.14$\pm$4.55 & 62.27$\pm$9.10 & \textbf{100.00$\pm$0.00} \\
 & REalTabFormer   & 94.39$\pm$1.09 & 96.26$\pm$0.50 & 93.45$\pm$0.09 & 97.06$\pm$1.01 \\
 & GReaT           & 91.31$\pm$1.61 & 92.46$\pm$0.99 & 85.91$\pm$1.40 & 99.08$\pm$0.75 \\
 & TabDDPM & \underline{96.05$\pm$2.22} & \textbf{96.94$\pm$0.77} & \textbf{94.74$\pm$0.00} & 99.14$\pm$1.53 \\
& CDTD      & 94.66$\pm$2.49 & 96.37$\pm$0.94 & \underline{94.47$\pm$1.18} & 98.28$\pm$1.77 \\
 & EPIC            & 94.67$\pm$1.96 & 96.06$\pm$1.35 & 93.42$\pm$2.34 & 98.71$\pm$0.77 \\
 & \textbf{RDDG (Ours)} & \textbf{96.58$\pm$1.27} & \underline{96.71$\pm$1.17} & 93.42$\pm$2.34 & \underline{99.95$\pm$0.00} \\
\bottomrule
\end{tabular}
}
\vspace{-6pt}
\caption{Imbalanced classification performance of different methods on four real datasets.}
\label{table1}
\end{minipage}
\hfill
\begin{minipage}{0.49\textwidth}
\centering
\resizebox{\textwidth}{!}{
\begin{tabular}{llcccc}
\toprule
\textbf{Dataset} & \textbf{Method} & \textbf{Macro-F1} & \textbf{BAL ACC} & \textbf{Sensitivity} & \textbf{Specificity} \\
\midrule
\multirow{7}{*}{\shortstack{Consumer\\Behavior}} 
& Original & 66.71$\pm$2.07 & 77.00$\pm$2.14 & 63.37$\pm$2.13 & 86.51$\pm$0.05 \\
& ADS-GAN & 65.77$\pm$1.12 & 77.53$\pm$1.08 & 61.27$\pm$1.04 & 84.43$\pm$1.03 \\
& REalTabFormer & 65.44$\pm$2.51 & 75.20$\pm$1.47 & 62.21$\pm$3.68 & 86.33$\pm$2.91 \\
& GReaT & 64.76$\pm$3.65 & \underline{77.80$\pm$2.33} & 57.07$\pm$1.30 & 80.82$\pm$3.76 \\
& TabDDPM & \underline{67.90$\pm$1.54} & 75.67$\pm$1.44 & 63.21$\pm$2.29 & \underline{88.90$\pm$3.40} \\
& CDTD      & 67.89$\pm$5.09 & 73.60$\pm$3.82 & 62.19$\pm$5.88 & 88.79$\pm$5.97 \\
& EPIC & 67.32$\pm$2.24 & 77.20$\pm$4.24 & \underline{63.88$\pm$2.10} & 87.20$\pm$3.44 \\
& \textbf{RDDG (Ours)} & \textbf{68.99$\pm$2.18} & \textbf{79.60$\pm$2.41} & \textbf{66.75$\pm$2.76} & \textbf{90.50$\pm$2.34} \\
\midrule
\multirow{7}{*}{\shortstack{Health\\Metrics}} 
& Original & 90.95$\pm$1.35 & 95.38$\pm$1.05 & 84.42$\pm$1.43 & 93.07$\pm$1.27 \\
& ADS-GAN & 89.26$\pm$1.64 & 95.15$\pm$2.79 & 88.83$\pm$3.82 & 94.15$\pm$3.45 \\
& REalTabFormer & 93.42$\pm$2.34 & 95.33$\pm$2.66 & 92.64$\pm$1.42 & 96.25$\pm$2.27 \\
& GReaT & 93.29$\pm$2.49 & \underline{96.12$\pm$2.57} & 92.72$\pm$1.36 & 96.40$\pm$1.78 \\
& TabDDPM & \textbf{96.09$\pm$0.99} & 96.00$\pm$1.22 & \textbf{96.08$\pm$0.98} & \textbf{98.14$\pm$0.69} \\
& CDTD      & \underline{95.49$\pm$0.78} & 95.50$\pm$0.95 & \underline{95.47$\pm$0.79} & \underline{97.77$\pm$0.53} \\
& EPIC & 93.89$\pm$2.41 & 95.10$\pm$2.89 & 91.56$\pm$1.90 & 95.56$\pm$2.21 \\
& \textbf{RDDG (Ours)} & 95.40$\pm$1.73 & \textbf{96.74$\pm$2.21} & 94.55$\pm$2.53 & 97.71$\pm$1.64 \\
\midrule
\multirow{6}{*}{\shortstack{Real\\Estate}} 
& Original & 80.06$\pm$1.32 & 83.97$\pm$1.28 & 79.75$\pm$1.44 & 92.75$\pm$1.17 \\
& ADS-GAN & 82.44$\pm$1.26 & 86.17$\pm$1.12 & 81.51$\pm$1.38 & \textbf{95.51$\pm$1.35} \\
& REalTabFormer & 81.18$\pm$2.26 & 84.97$\pm$2.45 & 80.74$\pm$2.54 & 90.74$\pm$2.19 \\
& GReaT & 85.16$\pm$1.26 & \underline{88.38$\pm$1.28} & 82.03$\pm$3.27 & 92.03$\pm$2.83 \\
& TabDDPM & 76.00$\pm$3.06 & 87.87$\pm$3.17 & \textbf{86.67$\pm$6.88} & 89.08$\pm$0.71 \\
& CDTD      & 72.98$\pm$3.41 & 85.34$\pm$3.06 & 81.90$\pm$6.29 & 88.77$\pm$0.54 \\
& EPIC & \underline{85.21$\pm$2.10} & 86.98$\pm$4.17 & 83.45$\pm$4.32 & 93.21$\pm$3.42 \\
& \textbf{RDDG (Ours)} & \textbf{88.70$\pm$1.72} & \textbf{88.50$\pm$1.12} & \underline{85.43$\pm$1.47} & \underline{95.21$\pm$2.10} \\
\midrule
\multirow{6}{*}{\shortstack{Social\\Network}} 
& Original & 87.87$\pm$2.15 & 96.15$\pm$2.34 & 83.63$\pm$2.67 & 98.18$\pm$3.01 \\
& ADS-GAN & 88.89$\pm$1.28 & 96.20$\pm$1.06 & 83.63$\pm$2.19 & 97.18$\pm$2.68 \\
& REalTabFormer & 95.69$\pm$3.09 & 98.00$\pm$3.17 & 95.52$\pm$2.06 & \textbf{99.49$\pm$3.58} \\
& GReaT & 96.88$\pm$3.36 & \underline{98.80$\pm$2.01} & 96.87$\pm$4.30 & 98.42$\pm$3.22 \\
& TabDDPM & 92.16$\pm$5.87 & 76.26$\pm$16.95 & 93.30$\pm$4.35 & 74.80$\pm$16.37 \\
& CDTD      & \textbf{98.99$\pm$0.89} & 97.40$\pm$1.45 & \textbf{99.00$\pm$0.89} & 97.16$\pm$1.71 \\
& EPIC & 86.87$\pm$2.21 & 96.30$\pm$2.57 & 83.76$\pm$3.13 & 98.19$\pm$3.29 \\
& \textbf{RDDG (Ours)} & \underline{97.12$\pm$2.13} & \textbf{98.89$\pm$2.73} & \underline{97.66$\pm$2.17} & \underline{99.24$\pm$3.21} \\
\bottomrule
\end{tabular}
}
\vspace{-1pt}
\caption{Imbalanced classification performance of different methods on four synthetic datasets.}
\label{table2}
\end{minipage}
\end{table*}

We evaluate imbalanced classification using weighted Macro-F1, Balanced Accuracy (BAL ACC), Sensitivity, and Specificity metrics. In Appendix \ref{Appendix-data}, we provide details about these evaluation metrics. As with EPIC, we use four mainstream classification models —XGBoost \citep{xg}, CatBoost \citep{cb}, LightGBM \citep{lg}, and GBDT — as baselines. The classifiers' performance is averaged to ensure robustness.

\noindent \textbf{Results on the Real-world Datasets}. Table \ref{table1} reports the classification results for four real-world datasets under imbalanced conditions. We observe that our proposed RDDG algorithm achieves strong BAL ACC and Sensitivity scores across most datasets. Moreover, in terms of the weighted Macro-F1 metric, RDDG also obtains the best performance on the Travel and Thyroid datasets. Heloc, unlike the other datasets, maintains class balance with approximately equal sample sizes. For such balanced datasets, BAL ACC is the most suitable evaluation metric, and RDDG outperforms all other methods on this metric. The only exception is the Sick dataset, where CDTD and TabDDPM obtain the best weighted Macro-F1 scores, yet RDDG still achieves the best Sensitivity score, which is an important criterion in imbalanced learning that assesses the performance on the minority classes.

Comparing RDDG with EPIC, we observe that on the imbalanced datasets, the average performance improvements are 2.27\% in weighted Macro-F1, 1.09\% in BAL ACC, 0.86\% in Sensitivity, and 1.91\% in Specificity, all of which are statistically significant. 

\noindent \textbf{Results on the Synthetic Datasets}. Table \ref{table2} reports the classification performance on four synthetic datasets with explicit inter-attribute correlations. In Table \ref{table2}, we observe that RDDG consistently obtains the best BAL ACC scores across all four datasets. It also achieves the best weighted Macro-F1 scores on the Consumer Behavior and Real Estate datasets, and ranks among the Top-2 on the Social Network dataset. The only exception is the Health Metrics dataset, in which TabDDPM and CDTD are the Top-2 performers, yet the gap between RDDG and CDTD is small. In terms of Sensitivity, RDDG is among the Top-2 performers on three out of the four datasets. On the Social Network dataset, we also observe a large variance in classification performance for TabDDPM, with performance varying substantially across classifiers, indicating that the utility of the generated samples differs across classification models. 

Comparing RDDG with EPIC, the average performance improvements are 2.04\% in weighted Macro-F1, 4.23\% in BAL ACC, 5.44\% in Sensitivity, and 2.13\% in Specificity.

In Tables \ref{table4} and \ref{table:epic_rddg} in Appendix \ref{Appendix-cls}, we also report the performance of EPIC and RDDG using other LLMs such as Llama 3.0 and Mistral Max, and the observations are generally consistent with GPT-3.5. Overall, these experiments validate the effectiveness of RDDG in generating high-quality samples for imbalanced classification.

\begin{figure*}[htbp]
    \centering
    \includegraphics[width=\textwidth]{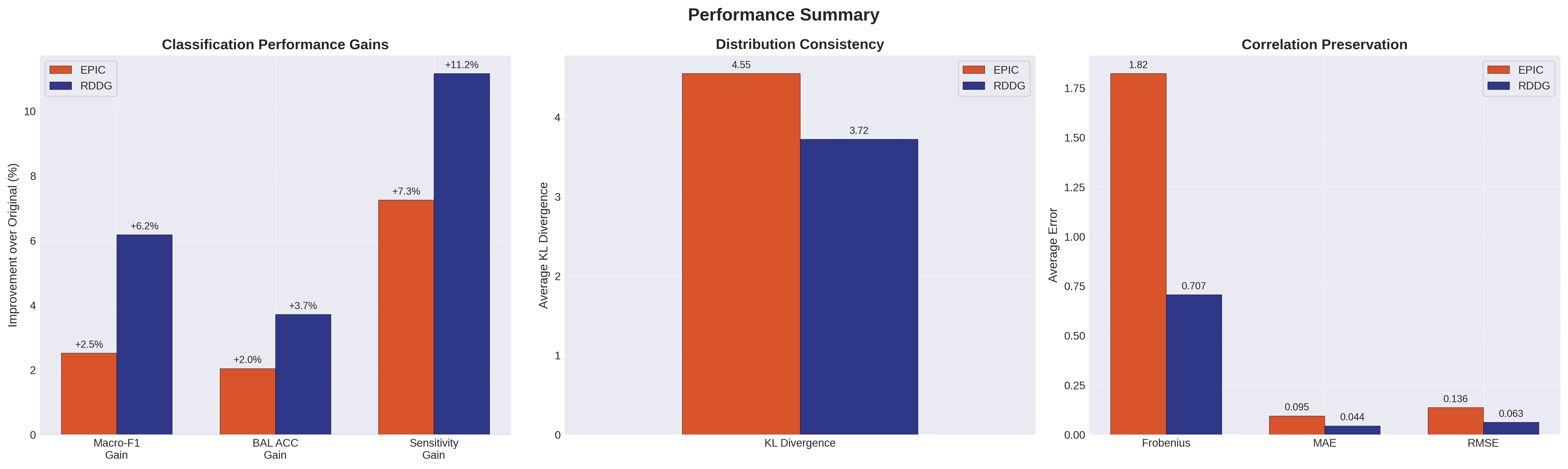}
    \vspace{-10pt}
    \caption{Overall performance summary comparing EPIC and RDDG across (a) classification performance gains (on the left, with higher values indicating better performance), (b) distribution consistency, and (c) correlation preservation metrics (center and right), where lower values indicate better performance.}
    \label{fig:overall_summary}
\end{figure*}

\begin{figure*}[htb!]
    \centering
    \includegraphics[width=\textwidth]{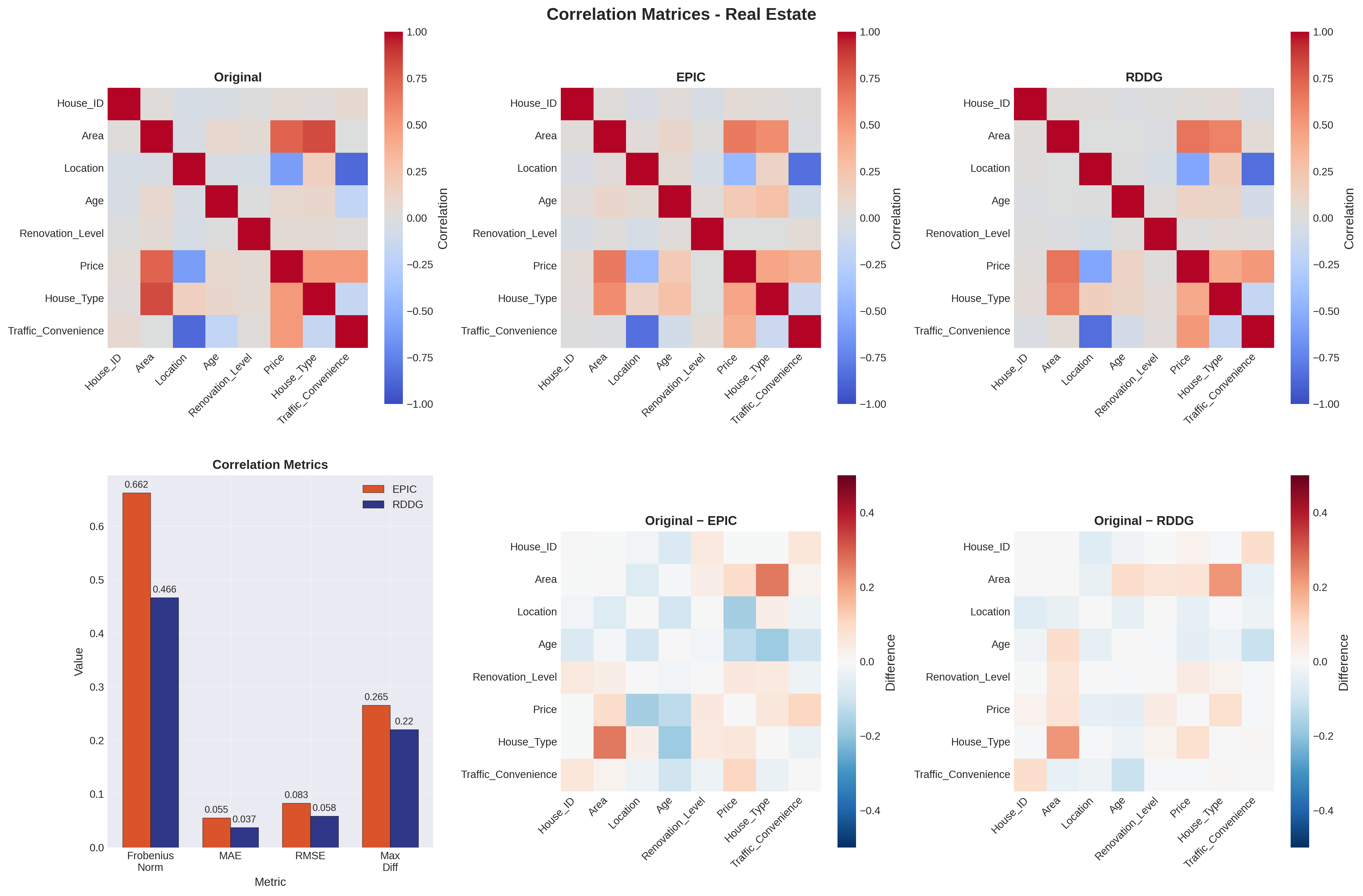}
    \vspace{-10pt}
    \caption{On the Real Estate dataset, RDDG demonstrates better correlation preservation than EPIC.}
    \label{fig:corr_realestate}
\end{figure*}

\subsubsection{Statistical Fidelity Evaluation}
To comprehensively evaluate the statistical fidelity of synthetic data generated by EPIC and RDDG, we examine both distribution consistency and inter-attribute correlation preservation, by employing Kullback-Leibler (KL) divergence for the former, and Frobenius norm, Mean Absolute Error (MAE) and Root Mean Square Error (RMSE) for the latter, more details are given in Appendix \ref{appendix:metrics}.

The overall performance summary (Figure~\ref{fig:overall_summary}) shows RDDG's consistent advantages across all evaluated metrics, outperforming EPIC on distribution consistency in 6 of 8 datasets and demonstrating a uniform superiority in correlation preservation. A comprehensive analysis, including detailed visualizations and dataset-specific results, is provided in Appendix~\ref{a:dist_consistency_corr}.

Specifically, RDDG achieves an 18.2\% reduction in overall KL divergence (3.72 vs. 4.55), indicating superior preservation of the original data's probabilistic structure (Figure~\ref{fig:kl_per_dataset} in ~\ref{a:dist_consistency_corr}). The method demonstrates particularly strong performance in complex datasets, with notable improvements in the Thyroid dataset (90.2\% reduction in KL divergence: 0.278 vs. 2.83) and Travel dataset (20.1\% reduction: 16.3 vs. 20.4). Distribution comparisons across representative datasets illustrate RDDG's superior ability to capture complex distributional shapes and modalities (Figure~\ref{fig:distributions_comparison}).

In terms of correlation preservation, RDDG exhibits remarkable performance with a 65.6\% improvement in overall Frobenius norm (0.827 vs. 2.40) and consistent advantages across all correlation metrics (Figures~\ref{fig:corr_thyroid} and~\ref{fig:corr_travel}). The MAE analysis indicates that RDDG preserves correlation structures, achieving 58.1\% higher accuracy than EPIC (0.048 vs. 0.115), and a 60.2\% lower RMSE (0.065 vs. 0.163).

In particular, on datasets with explicit inter-attribute correlations, such as Real Estate, we show that RDDG better captures and preserves inter-attribute correlations than EPIC, as shown in Figure \ref{fig:corr_realestate}. Additional fidelity analysis results are provided in Appendix \ref{a:dist_consistency_corr}. 

\subsection{Ablation Studies}

\paragraph{Effects of Core Set and Feedback Mechanism} To investigate the effect of the core set, we compare classification performance with randomly sampled subsets, on Travel and Thyroid. As shown in Table \ref{tab:travel_thyroid_results} in Appendix \ref{Appendix-ablation}, the core set algorithm achieves a substantial performance gain compared to randomly sampled subsets. Moreover, we also show the influence of our self-reinforcing feedback mechanism, which significantly improves the overall performance.  

\paragraph{Influence of Imbalance Ratio} To study the influence of imbalance ratio (IR) on the performance of EPIC and RDDG, we use the UCI segmentation dataset and artificially balance the ratios between the numbers of majority and minority class samples via an exponentially decaying strategy. As reported in Table \ref{ir_result} in Appendix \ref{Appendix-ablation}, RDDG consistently outperforms EPIC across all metrics for varying IR values. When the IR is high, the advantage of RDDG over EPIC is more substantial. 

\section{Conclusion}
In this work, we propose a dynamically guided in-context tabular data synthesis framework that comprises progressive CoT steps and a self-reinforcing feedback mechanism. By integrating the explicit functional dependency constraints discovered through relation mining and the dynamic feedback mechanism, our framework significantly improves both data fidelity and downstream imbalanced classification performance.

\section*{Limitations}

While RDDG demonstrates substantial improvements in tabular data synthesis, several limitations warrant consideration. First, our approach requires either external API calls or local deployment of a foundation model. Second, as an in-context learning framework, although it is training-free, our approach is bounded by the inherent capability of LLMs. Third, LLM token limitations constrain the volume of examples that can be processed simultaneously, necessitating batch-wise generation for large-scale synthesis tasks. However, this may affect the global consistency of generated samples across batches. 

\section*{Acknowledgments}
This work was partially supported by the MOE Liberal Arts and Social Sciences Foundation (No.~23YJAZH210), the Major Program of the National Social Science Foundation (No.~23\&ZD309), the Henan Provincial Center for Outstanding Overseas Scientists (No.~GZS2025004), and the High-Level Talent International Training Program of Henan Province (No.~GCC2025010). Julian Rodemann acknowledges funding support from the Federal Statistical Office of Germany within the cooperation project ``Machine Learning in Official Statistics'', as well as from the Bavarian Institute for Digital Transformation (bidt) and the Bavarian Academy of Sciences and Humanities (BAdW) through a graduate scholarship. Esteban Garcés Arias acknowledges support from the Mentoring Program at the Faculty of Mathematics, Informatics, and Statistics at LMU Munich and from the MCML (Munich Center for Machine Learning). We thank the anonymous reviewers, area chairs, and program committee members for their constructive feedback.

\section*{Ethics Statement}

We affirm that our research adheres to the \href{https://www.aclweb.org/portal/content/acl-code-ethics}{ACL Ethics Policy}. This work uses publicly available datasets and contains no personally identifiable information. We declare that there are no conflicts of interest that could potentially influence the outcomes, interpretations, or conclusions of this research. All funding sources supporting this study are acknowledged in the acknowledgments section. We have diligently documented our methodology, experiments, and results, and we commit to sharing our code, data, and other relevant resources to enhance reproducibility.

\bibliography{custom}

\clearpage

\onecolumn

\appendix

\section{Appendix}
\label{sec:appendix}

\subsection{Additional Imbalanced Classification Results} \label{Appendix-cls}

In our experiments, GPT-3.5 (GPT-3.5-turbo-0125) is used as our default LLM. While the original EPIC study \citep{epic} used GPT-3.5-turbo-0613 and GPT-3.5-turbo-16k-0613 models, in this work, we evaluate EPIC and our approach using GPT-3.5-turbo-0125, as it is the model currently available to us\footnote{https://platform.openai.com/docs/deprecations/2023-11-06-chat-model-updates. As of June 17, 2024, GPT-3.5-turbo-0613 and GPT-3.5-turbo-16k-0613 have been deprecated by OpenAI, and ``only existing users of these models will be able to continue using them''.}. In addition to GPT-3.5, we also report the performance of EPIC and RDDG using other LLMs, such as Llama 3.0 and Mistral Max, and present the results below. For each dataset, each method will generate \textit{1000} synthetic samples (the target threshold). 

\noindent \textbf{Impact of LLM Choice on the Imbalanced Classification Performance}. To investigate the impact of LLM choice on RDDG performance, we conduct additional classification experiments that incorporate the open-source LLMs Llama 3.0 and Mistral Max, as well as the proprietary GPT-3.5 model used in our approach. As shown in Table \ref{table4}, we evaluate our RDDG framework using these alternative LLMs for comparison. A comparative analysis reveals that Mistral Max underperforms the other two LLMs in classification accuracy, with Llama 3.0 slightly trailing GPT-3.5. In Table \ref{table:epic_rddg}, we also compare the performance of EPIC and RDDG under Llama 3.0 and Mistral Max. We observe that RDDG consistently outperforms EPIC across both LLMs, except for Thyroid when using Mistral. Overall, regardless of LLM choice, RDDG outperforms EPIC.

\begin{table*}[htbp] 
\centering
\resizebox{0.55\textwidth}{!}{
\begin{tabular}{llccccc}
\toprule
\textbf{Dataset} & \textbf{Method} & \textbf{\#syn} & \textbf{Macro-F1} & \textbf{BAL ACC} & \textbf{Sensitivity} & \textbf{Specificity} \\
\midrule
\multirow{4}{*}{Travel}
 & Original        & -  & 58.12$\pm$2.04 & 71.21$\pm$1.56 & 58.12$\pm$2.04 & \textbf{85.63 $\pm$ 0.85} \\
 & Mistral         & 1K  & 66.21$\pm$1.32 & 78.01$\pm$1.72 & 77.21$\pm$1.02 & 76.21$\pm$1.99 \\
 & Llama 3.0    & 1K  & 65.32$\pm$1.02 & 77.23$\pm$2.35 & 76.97$\pm$0.71 & 77.15$\pm$1.56 \\
 & GPT-3.5           & 1K  & \textbf{68.51$\pm$2.11} & \textbf{79.52$\pm$3.16} & \textbf{79.16$\pm$0.11} & 83.55$\pm$2.64 \\
\midrule
\multirow{4}{*}{Sick}
 & Original        & -  & 87.82$\pm$2.46 & 91.22$\pm$0.93 & 82.84$\pm$1.76 & \textbf{99.61 $\pm$ 0.28} \\
 & Mistral         & 1K  & 88.04$\pm$1.36 & 94.62$\pm$1.53 & \textbf{90.23$\pm$3.35} & 99.02$\pm$0.36 \\
 & Llama 3.0    & 1K  & \textbf{88.76$\pm$1.67} & \textbf{94.67$\pm$0.50} & 90.22$\pm$1.12 & 99.13$\pm$0.28 \\
 & GPT-3.5           & 1K  & 87.99$\pm$0.91 & 93.62$\pm$0.95 & 88.04 $\pm$1.93 & 99.20$\pm$0.07 \\
\midrule
\multirow{4}{*}{Heloc}
 & Original        & -  & \textbf{71.01$\pm$0.47} & 72.21$\pm$0.37 & 67.19$\pm$0.86 & \textbf{78.52$\pm$0.38} \\
 & Mistral         & 1K  & 70.46$\pm$0.55 & 72.53$\pm$0.38 & 68.25$\pm$0.91 & 76.81$\pm$0.19 \\
 & Llama 3.0    & 1K  & 70.62$\pm$0.56 & \textbf{72.63$\pm$0.41} & \textbf{68.53$\pm$0.89} & 76.74$\pm$0.29 \\
 & GPT-3.5           & 1K  & 70.32$\pm$0.55 & 72.54$\pm$0.43 & 67.72$\pm$0.78 & 77.35$\pm$0.13 \\
\midrule
\multirow{4}{*}{Thyroid}
 & Original        & -  & 94.23$\pm$1.99 & 95.08$\pm$1.60 & 91.14$\pm$3.12 & 99.02$\pm$1.01 \\
 & Mistral         & 1K  & 93.23$\pm$3.15 & 94.75$\pm$2.29 & 90.79$\pm$4.48 & 98.71$\pm$1.47 \\
 & Llama 3.0    & 1K  & 94.06$\pm$2.22 & 95.85$\pm$1.33 & 93.42$\pm$2.34 & 98.28$\pm$1.25 \\
 & GPT-3.5           & 1K  & 
 \textbf{96.58$\pm$1.27} & \textbf{96.71$\pm$1.17} & \textbf{93.42$\pm$2.34} & \textbf{99.95$\pm$0.00} \\
\bottomrule
\end{tabular}
}
\vspace{1em}
\caption{Ablation study on the impact of LLM choice on RDDG performance.}
\label{table4}
\end{table*}

\begin{table*}[htbp]
\centering
\resizebox{0.8\textwidth}{!}{
\begin{tabular}{llcccccccc}
\toprule
\multirow{2}{*}{\textbf{Dataset}} & \multirow{2}{*}{\textbf{Method}} 
& \multicolumn{4}{c}{\textbf{Llama}} 
& \multicolumn{4}{c}{\textbf{Mistral}} \\
\cmidrule(lr){3-6} \cmidrule(lr){7-10}
& & \textbf{Macro-F1} & \textbf{BAL ACC} & \textbf{Sensitivity} & \textbf{Specificity} 
  & \textbf{Macro-F1} & \textbf{BAL ACC} & \textbf{Sensitivity} & \textbf{Specificity} \\
\midrule
\multirow{2}{*}{Travel}
 & EPIC  & 63.24$\pm$1.22 & 74.23$\pm$3.20 & 74.67$\pm$1.21 & 75.25$\pm$1.29
         & 64.35$\pm$0.98 & 75.32$\pm$2.21 & 76.78$\pm$1.22 & 75.67$\pm$1.48 \\
 & RDDG  & \textbf{65.32$\pm$1.02} & \textbf{77.23$\pm$2.35} & \textbf{76.97$\pm$0.71} & \textbf{77.15$\pm$1.56}
         & \textbf{66.21$\pm$1.32} & \textbf{78.01$\pm$1.72} & \textbf{77.21$\pm$1.02} & \textbf{76.21$\pm$1.99} \\
\midrule
\multirow{2}{*}{Sick}
 & EPIC  & 84.41$\pm$1.36 & 93.44$\pm$0.84 & 88.26$\pm$1.64 & 98.62$\pm$0.13
         & 86.43$\pm$1.05 & 94.50$\pm$0.49 & 90.22$\pm$1.12 & 98.77$\pm$0.22 \\
 & RDDG  & \textbf{88.76$\pm$1.67} & \textbf{94.67$\pm$0.50} & \textbf{90.22$\pm$1.12} & \textbf{99.13$\pm$0.28}
         & \textbf{88.04$\pm$1.36} & \textbf{94.62$\pm$1.53} & \textbf{90.22$\pm$3.35} & \textbf{99.02$\pm$0.36} \\
\midrule
\multirow{2}{*}{Heloc}
 & EPIC  & 70.18$\pm$0.30 & 72.29$\pm$0.23 & 67.91$\pm$0.57 & 76.66$\pm$0.49
         & 70.14$\pm$0.16 & 72.40$\pm$0.09 & 67.47$\pm$0.69 & \textbf{77.32$\pm$0.80} \\
 & RDDG  & \textbf{70.62$\pm$0.56} & \textbf{72.63$\pm$0.41} & \textbf{68.53$\pm$0.89} & \textbf{76.74$\pm$0.29}
         & \textbf{70.46$\pm$0.55} & \textbf{72.53$\pm$0.38} & \textbf{68.25$\pm$0.91} & 76.81$\pm$0.19 \\
\midrule
\multirow{2}{*}{Thyroid}
 & EPIC  & 88.96$\pm$2.54 & 91.24$\pm$0.88 & 84.21$\pm$0.00 & 98.28$\pm$1.77
         & \textbf{95.38$\pm$1.14} & \textbf{96.72$\pm$0.38} & \textbf{94.74$\pm$0.00} & \textbf{98.71$\pm$0.77} \\
 & RDDG  & \textbf{94.06$\pm$2.22} & \textbf{95.85$\pm$1.33} & \textbf{93.42$\pm$2.34} & \textbf{98.28$\pm$1.25}
         & 93.23$\pm$3.15 & 94.75$\pm$2.29 & 90.79$\pm$4.48 & 98.71$\pm$1.47 \\
\bottomrule
\end{tabular}}
\vspace{1em}
\caption{Comparison of EPIC and RDDG under Llama 3.0 and Mistral Max.}
\label{table:epic_rddg}
\end{table*}

\begin{table*}[htbp]
\centering
\resizebox{0.9\textwidth}{!}{
\begin{tabular}{l l c c c c c c c c c c}
\hline
\textbf{Dataset} & \textbf{Method} & \textbf{Macro-F1} & \textbf{BAL ACC} & \textbf{Sensitivity} & \textbf{Specificity} & \textbf{Token} & \textbf{Expenses} & \textbf{Prompt 1} & \textbf{Prompt 2} & \textbf{Prompt 3} & \textbf{All} \\
\hline
\multirow{2}{*}{Consumer Behavior} 
  & EPIC  & 67.32$\pm$2.24 & 77.20$\pm$4.24 & 63.88$\pm$2.10 & 87.20$\pm$3.44 & 186K & \$0.13 & - & - & - & 76.0s \\
  & RDDG  & 68.99$\pm$2.18 & 79.60$\pm$2.41 & 66.75$\pm$2.76 & 90.50$\pm$2.34 & 176K & \$0.12 & 6.0s & 2.2s & 137.3s & 169.7s \\
\hline
\multirow{2}{*}{Health Metrics} 
  & EPIC  & 93.89$\pm$2.41 & 95.10$\pm$2.89 & 91.56$\pm$1.90 & 95.56$\pm$2.21 & 475K & \$0.43 & - & - & - & 96.3s \\
  & RDDG  & 95.40$\pm$1.73 & 96.74$\pm$2.21 & 94.55$\pm$2.53 & 97.71$\pm$1.64 & 493K & \$0.34 & 6.4s & 2.7s & 47.6s & 78.7s \\
\hline
\multirow{2}{*}{Real Estate} 
  & EPIC  & 85.21$\pm$2.10 & 86.98$\pm$1.47 & 83.45$\pm$4.32 & 93.21$\pm$4.02 & 150K & \$0.10 & - & - & - & 35.2s \\
  & RDDG  & 88.70$\pm$1.72 & 88.50$\pm$1.12 & 85.43$\pm$1.47 & 95.21$\pm$2.10 & 170K & \$0.10 & 5.7s & 2.4s & 37.8s & 66.7s \\
\hline
\multirow{2}{*}{Social Network} 
  & EPIC  & 86.87$\pm$2.21 & 96.30$\pm$2.57 & 83.76$\pm$3.13 & 98.19$\pm$3.29 & 330K & \$0.15 & - & - & - & 32.7s \\
  & RDDG  & 97.12$\pm$2.13 & 98.89$\pm$2.73 & 97.66$\pm$2.17 & 99.24$\pm$3.21 & 240K & \$0.17 & 6.0s & 2.7s & 38.9s & 58.5s \\
\hline
\end{tabular}}
\caption{Comparison of EPIC and RDDG in terms of accuracy performances, number of input and output tokens, API expenses, and running time.}
\label{tab:epic_rddg_results}
\end{table*}

\noindent \textbf{LLM Cost}. While LLM cost is a valid concern, our findings in Table \ref{tab:epic_rddg_results} show that it remains manageable in practice. ``Expenses'' denotes the API cost for each dataset. For instance, generating 1,000 new samples costs under 0.5\$ using the GPT API. ``Token'' represents the total number of input and output tokens for each approach. Note that output tokens typically cost three to five times more than input tokens, as generating responses is computationally more intensive than processing prompts. It is also worth noting that both EPIC and RDDG remove the requirement for computationally expensive model training. That is, they enable users to generate synthetic samples directly through the GPT API or locally deployed open-source LLMs, eliminating the need for model training or fine-tuning. 

\noindent \textbf{Time Efficiency}. In Table \ref{tab:epic_rddg_results}, we also report the total and specific time consumption of different prompts in RDDG, in which we find that Prompt 3 consumes substantially more time since it involves batch-wise iterative data generation and a self-reinforcing feedback mechanism. Despite the increased time required, RDDG consistently outperforms EPIC, producing higher-quality synthetic data. Overall, because both methods eliminate the need for computationally expensive model training, they exhibit outstanding efficiency (under 170 seconds) compared to other non-in-context learning approaches, which typically require multiple hours or even days of model training or fine-tuning.

\subsection{Fidelity Analysis}
\label{a:dist_consistency_corr}

To evaluate the fidelity of the synthetic data generated by EPIC and RDDG methods, we conduct an analysis focusing on two critical aspects: (i) distribution consistency of the synthetic data with respect to the original data, and (ii) preservation of inter-attribute correlations. The evaluations are performed across all datasets using standardized preprocessing and normalization procedures.

\subsubsection{Evaluation Metrics for Statistical Fidelity}
\label{appendix:metrics}

We employ several metrics to assess the statistical fidelity of synthetic data along two critical dimensions: distributional consistency and preservation of inter-attribute correlations. For distribution consistency, we employ the Kullback-Leibler (KL) divergence as our primary metric, which quantifies the information loss incurred when approximating the original distribution. For inter-attribute correlation preservation, we first calculate the inter-attribute Pearson correlation coefficients in the original data and the synthetic data generated by EPIC and RDDG, respectively, and then derive the correlation difference matrix (taking absolute value) between the correlation coefficient matrices of the original data and the synthetic data generated by EPIC and RDDG, respectively. We then use four metrics: Frobenius norm, Mean Absolute Error (MAE), Root Mean Square Error (RMSE), and Max Difference to measure the overall preservation of inter-attribute correlations in the correlation difference matrix. Below, we provide formal definitions of each metric.

\paragraph{Kullback-Leibler (KL) Divergence}
The KL divergence quantifies the information loss when approximating the original data distribution $P$ with the synthetic data distribution $Q$. For continuous attributes, we discretize the data into $k$ bins and compute:

\begin{equation}
D_{KL}(P \| Q) = \sum_{i=1}^{k} P(i) \log \frac{P(i)}{Q(i)}
\end{equation}

\noindent where $P(i)$ and $Q(i)$ represent the probability mass in bin $i$ for the original and synthetic distributions, respectively. In our experiments, we use $k=50$ bins with equal-width binning after standardization. The overall KL divergence for a dataset is computed as the mean across all numeric attributes:

\begin{equation}
\overline{D}_{KL} = \frac{1}{n} \sum_{j=1}^{n} D_{KL}(P_j \| Q_j)
\end{equation}

\noindent where $n$ is the number of numeric attributes. Lower values indicate better preservation of the distribution, with $D_{KL} = 0$ indicating a perfect match.

\paragraph{Frobenius Norm}
Let $\mathbf{C}_{real} \in \mathbb{R}^{n \times n}$ and $\mathbf{C}_{syn} \in \mathbb{R}^{n \times n}$ denote the Pearson correlation matrices for the original and synthetic data, respectively, where each element $C_{ij}$ represents the correlation coefficient between attributes $i$ and $j$.

The Frobenius norm measures the overall magnitude of differences between correlation matrices (i.e., the correlation difference matrix):

\begin{equation}
\|\mathbf{C}_{real} - \mathbf{C}_{syn}\|_F = \sqrt{\sum_{i=1}^{n} \sum_{j=1}^{n} |C_{real,ij} - C_{syn,ij}|^2}
\end{equation}

\noindent This metric provides a single scalar value that quantifies the total deviation in the correlation structure.

\paragraph{Mean Absolute Error (MAE)}
The MAE computes the average absolute difference across all pairwise correlations:

\begin{equation}
\text{MAE} = \frac{1}{n^2} \sum_{i=1}^{n} \sum_{j=1}^{n} |C_{real,ij} - C_{syn,ij}|
\end{equation}

\noindent This metric is more interpretable than the Frobenius norm, as it represents the average deviation of the correlation coefficients.

\paragraph{Root Mean Square Error (RMSE)}
The RMSE penalizes larger deviations more heavily than MAE:

\begin{equation}
\text{RMSE} = \sqrt{\frac{1}{n^2} \sum_{i=1}^{n} \sum_{j=1}^{n} (C_{real,ij} - C_{syn,ij})^2}
\end{equation}

\paragraph{Maximum Absolute Difference (Max Diff)}
The maximum absolute difference identifies the worst-case correlation preservation:

\begin{equation}
\text{Max Diff} = \max_{i,j} |C_{real,ij} - C_{syn,ij}|
\end{equation}

\noindent This metric helps identify worst-case scenarios for specific attribute pairs where correlation preservation fails substantially.

For all metrics described above, lower values indicate better preservation of statistical properties. Specifically, in our evaluations:
\begin{itemize}
    \item \textbf{KL Divergence}: Values near 0 indicate excellent distribution matching; values above 1.0 suggest substantial distributional differences.
    \item \textbf{Correlation Metrics}: For datasets with $n$ attributes, perfect correlation preservation yields 0 for all metrics. As a reference, random synthetic data typically produces Frobenius norm values of $O(\sqrt{n})$.
    \item \textbf{Comparative Analysis}: We report both absolute metrics and relative improvements (percentage reduction) compared to baseline methods to contextualize performance gains.
\end{itemize}

\subsubsection{Distribution Consistency}

Within each dataset, we compute the mean KL divergence over numeric attributes and compute macro-averages across datasets using equal weighting. Figure~\ref{fig:kl_per_dataset} presents the comparative KL divergence analysis across all eight datasets. We observe that RDDG outperforms EPIC on 6 of 8 datasets. Notable improvements are observed on the Travel dataset (16.3 vs. 20.4, a 20.1\% reduction), the Thyroid dataset (KL divergence: 0.278 for RDDG vs. 2.83 for EPIC, a 90.2\% reduction), and the Consumer Behavior dataset (2.37 vs. 3.5, a 32.3\% reduction). The Heloc and Social Network datasets are exceptions where EPIC demonstrates better performance. Overall, RDDG demonstrates superior distributional consistency with the original data, with a macro-average KL divergence of 3.72 compared to EPIC's 4.55, representing an 18.2\% improvement.

\begin{figure}[htb!]
    \centering
    \includegraphics[width=\textwidth]{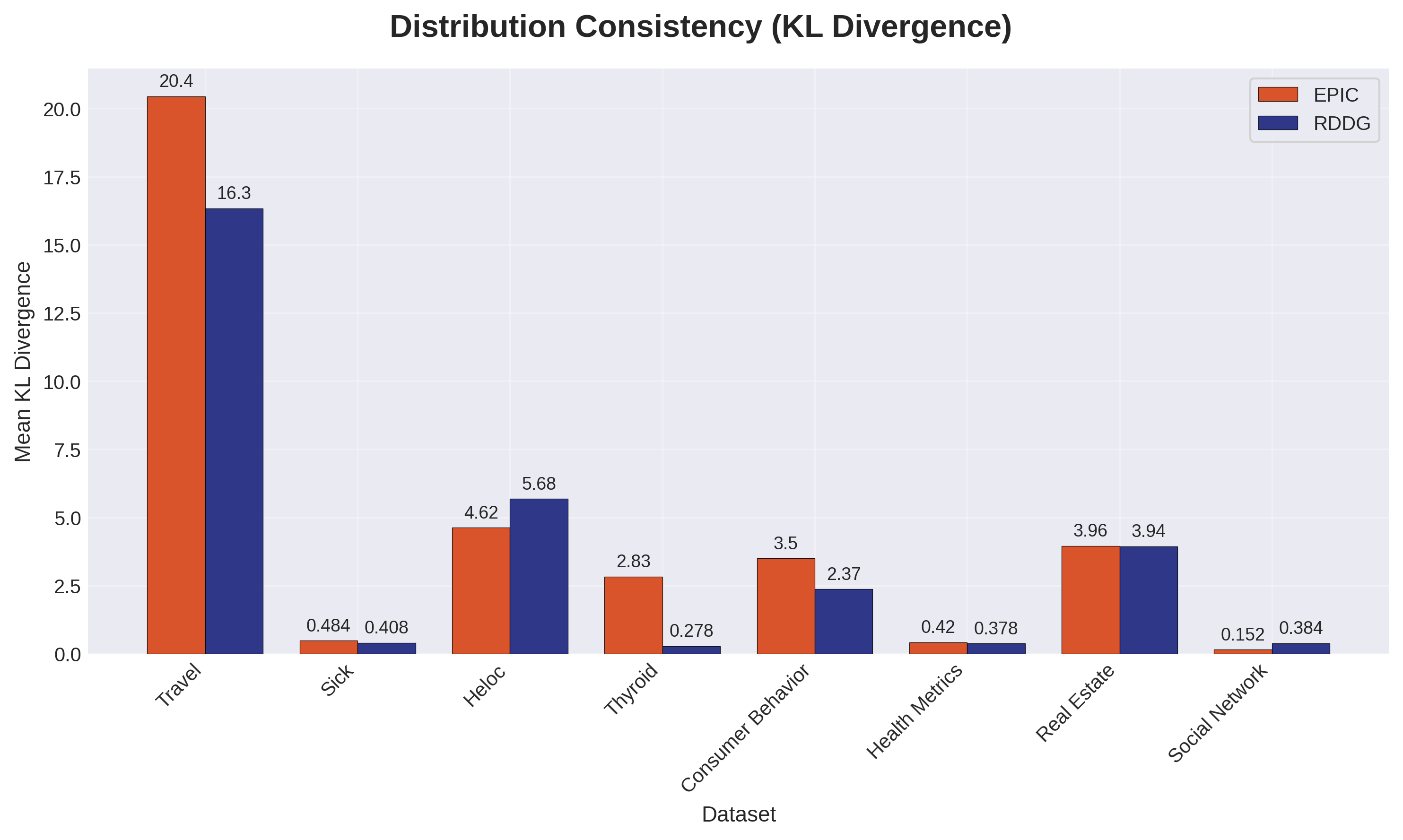}
    \caption{Mean KL divergence per dataset comparing EPIC and RDDG methods. Lower values indicate better distribution preservation.}
    \label{fig:kl_per_dataset}
\end{figure}

\begin{figure}[htb!]
    \centering
    \begin{subfigure}[b]{0.32\textwidth}
        \includegraphics[width=\textwidth]{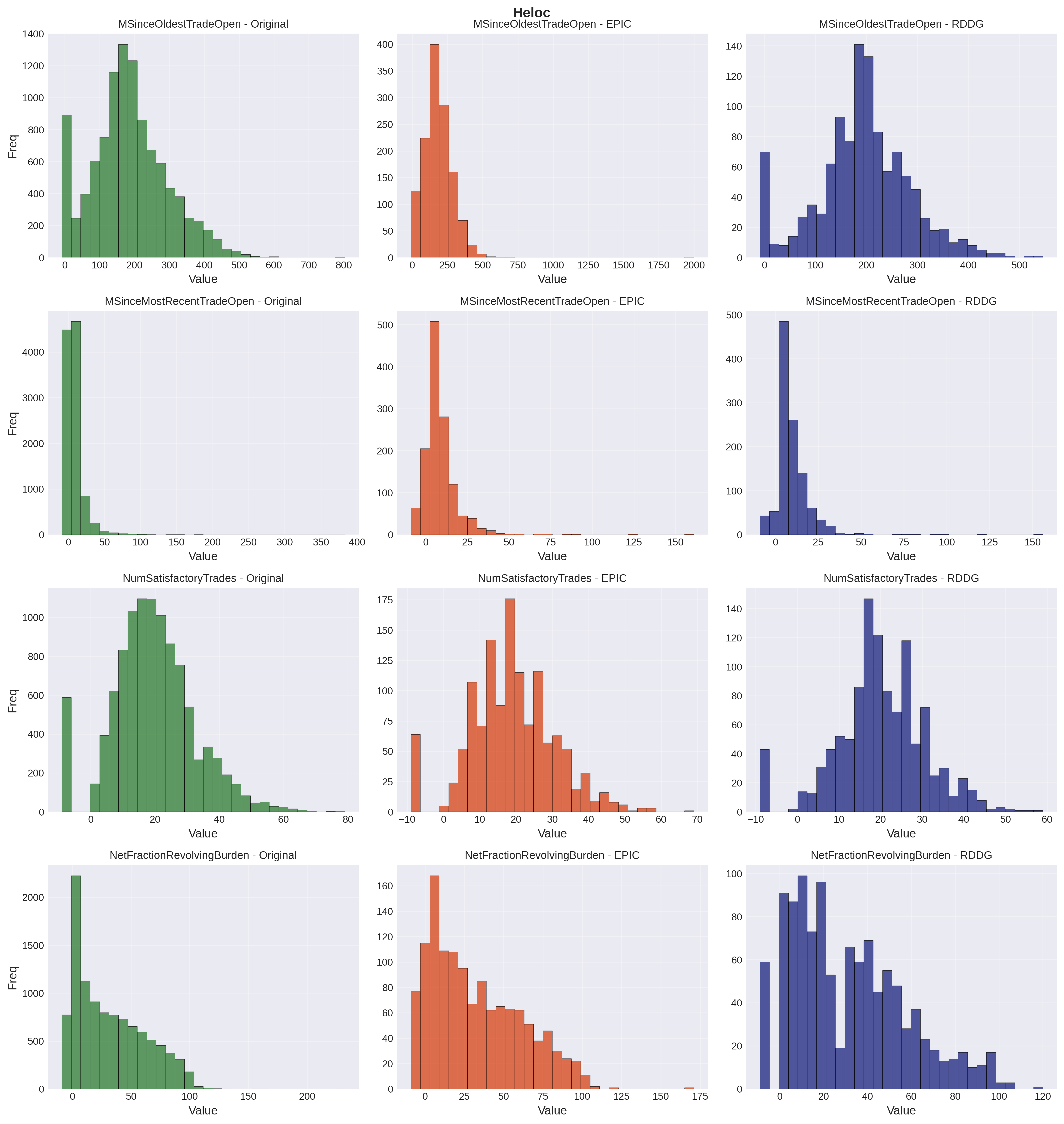}
        \caption{Heloc}
        \label{fig:dist_heloc}
    \end{subfigure}
    \begin{subfigure}[b]{0.32\textwidth}
        \includegraphics[width=\textwidth]{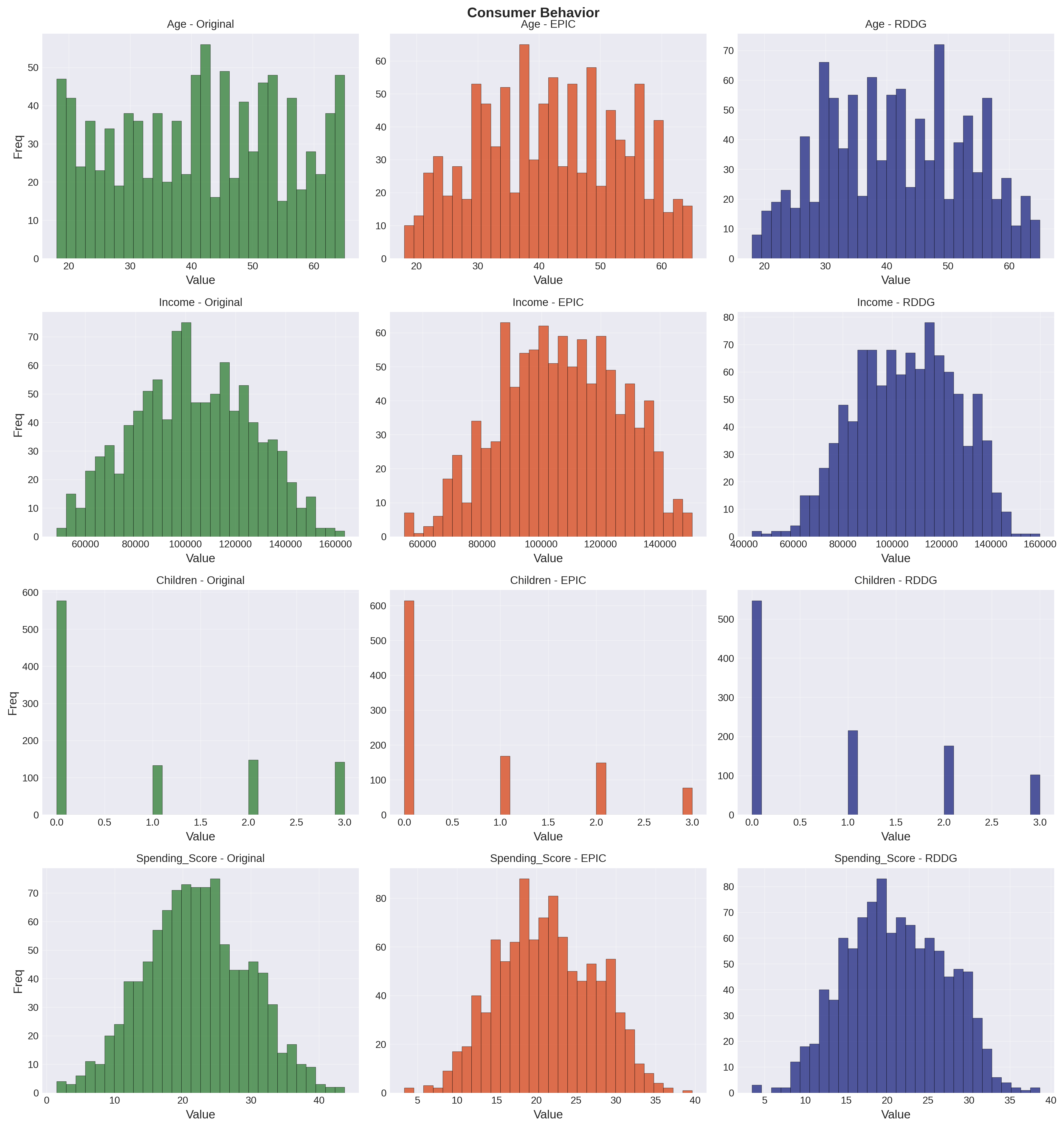}
        \caption{Consumer Behavior}
        \label{fig:dist_consumer}
    \end{subfigure}
    \begin{subfigure}[b]{0.32\textwidth}
        \includegraphics[width=\textwidth]{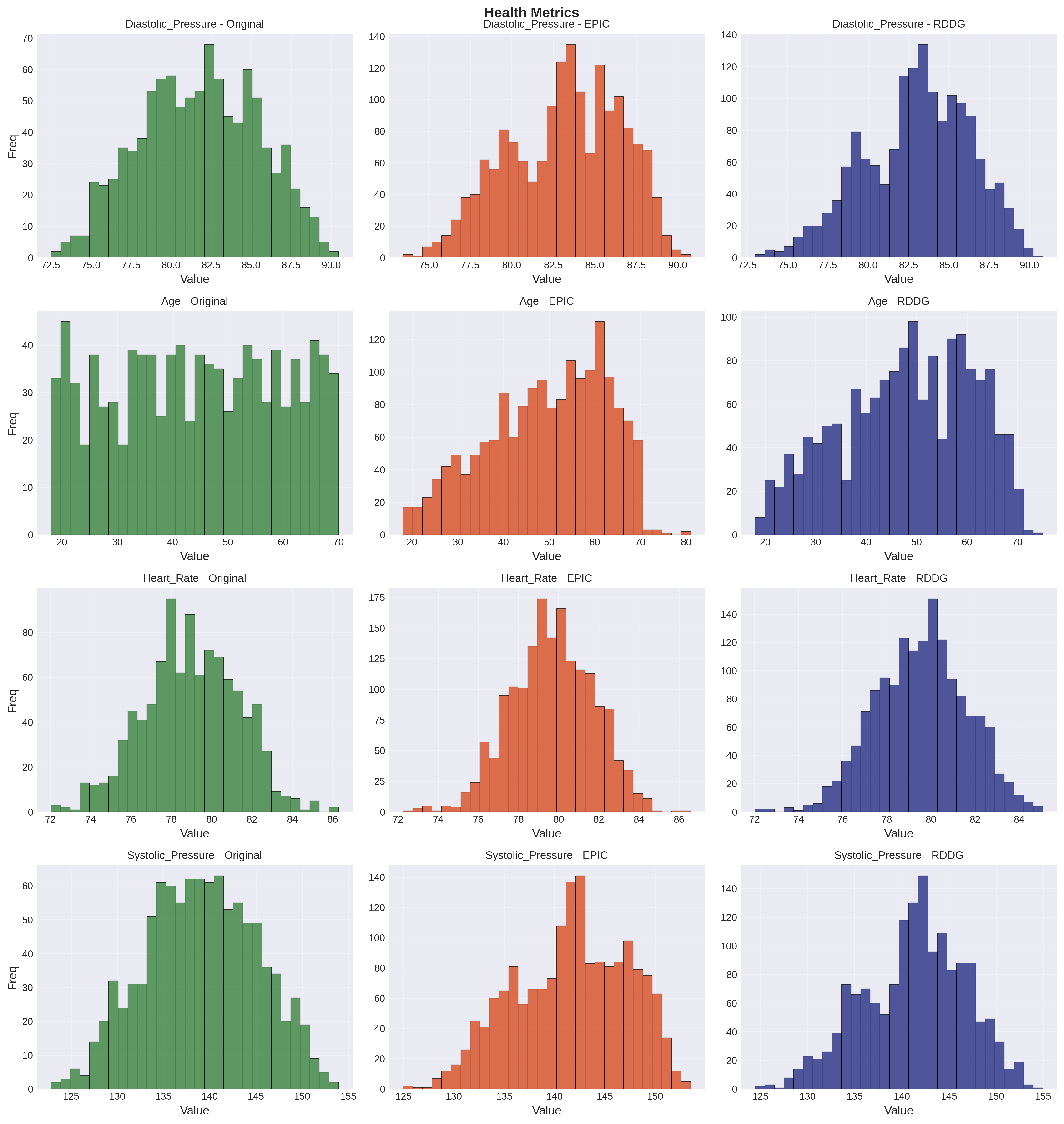}
        \caption{Health Metrics}
        \label{fig:dist_health}
    \end{subfigure}
    \caption{Distribution comparisons between original data, and synthetic data generated by both EPIC and RDDG, respectively, over selected features/attributes across three representative datasets.}
    \label{fig:distributions_comparison}
\end{figure}

To provide detailed insights into distribution preservation, we visualize the distributions on representative datasets in Figures~\ref{fig:dist_heloc}-- \ref{fig:dist_health}. We observe that the synthetic data generated by both EPIC and RDDG generally follows the distribution of the attributes in the original datasets. 

\begin{figure}[htb!]
    \centering
    \includegraphics[width=\textwidth]{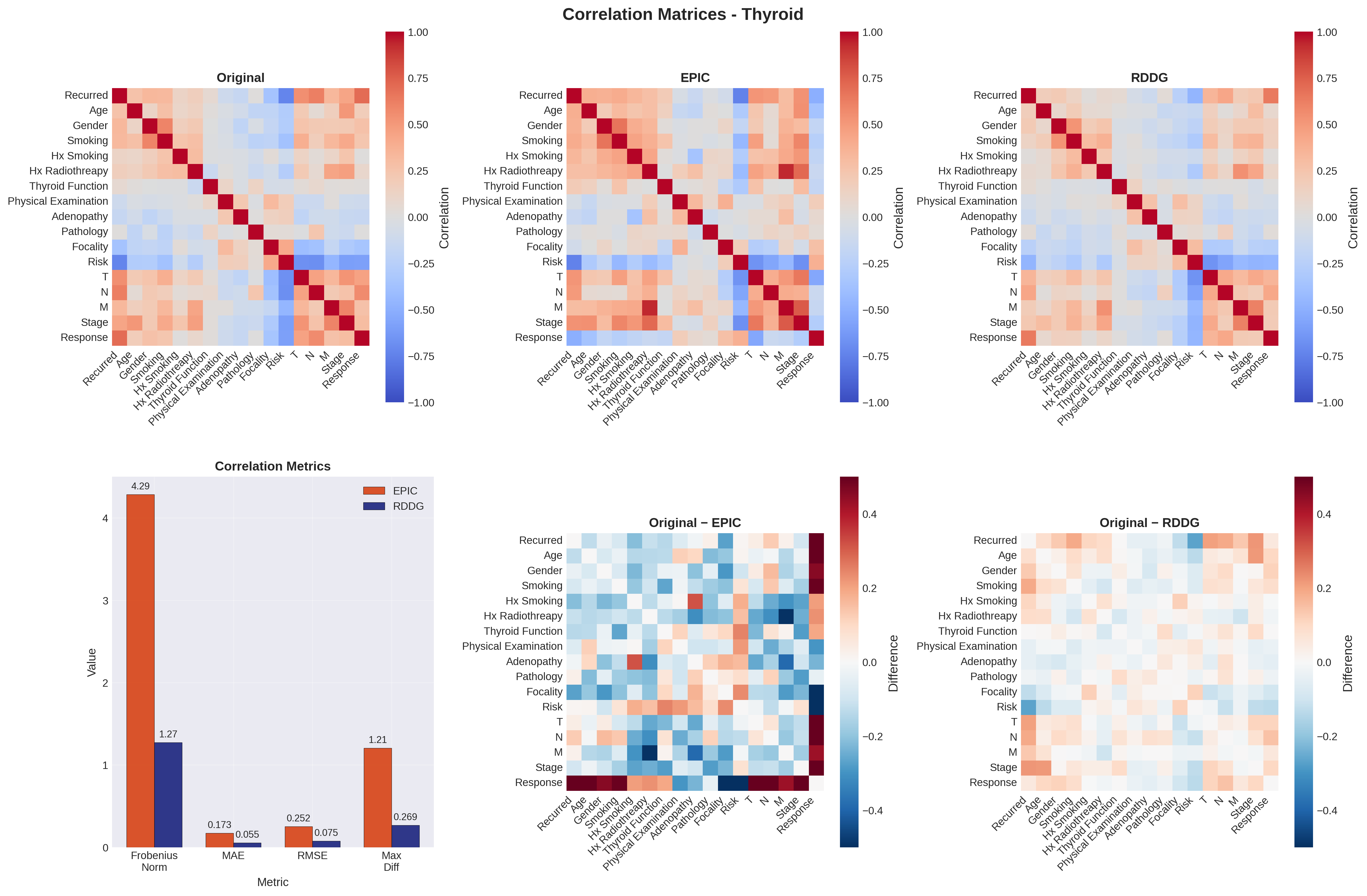}
    \caption{Correlation matrix analysis for the Thyroid dataset showing original correlations, synthetic data correlations (EPIC and RDDG), difference heatmaps, and preservation metrics comparison.}
    \label{fig:corr_thyroid}
\end{figure}

\begin{figure}[htb!]
    \centering
    \includegraphics[width=\textwidth]{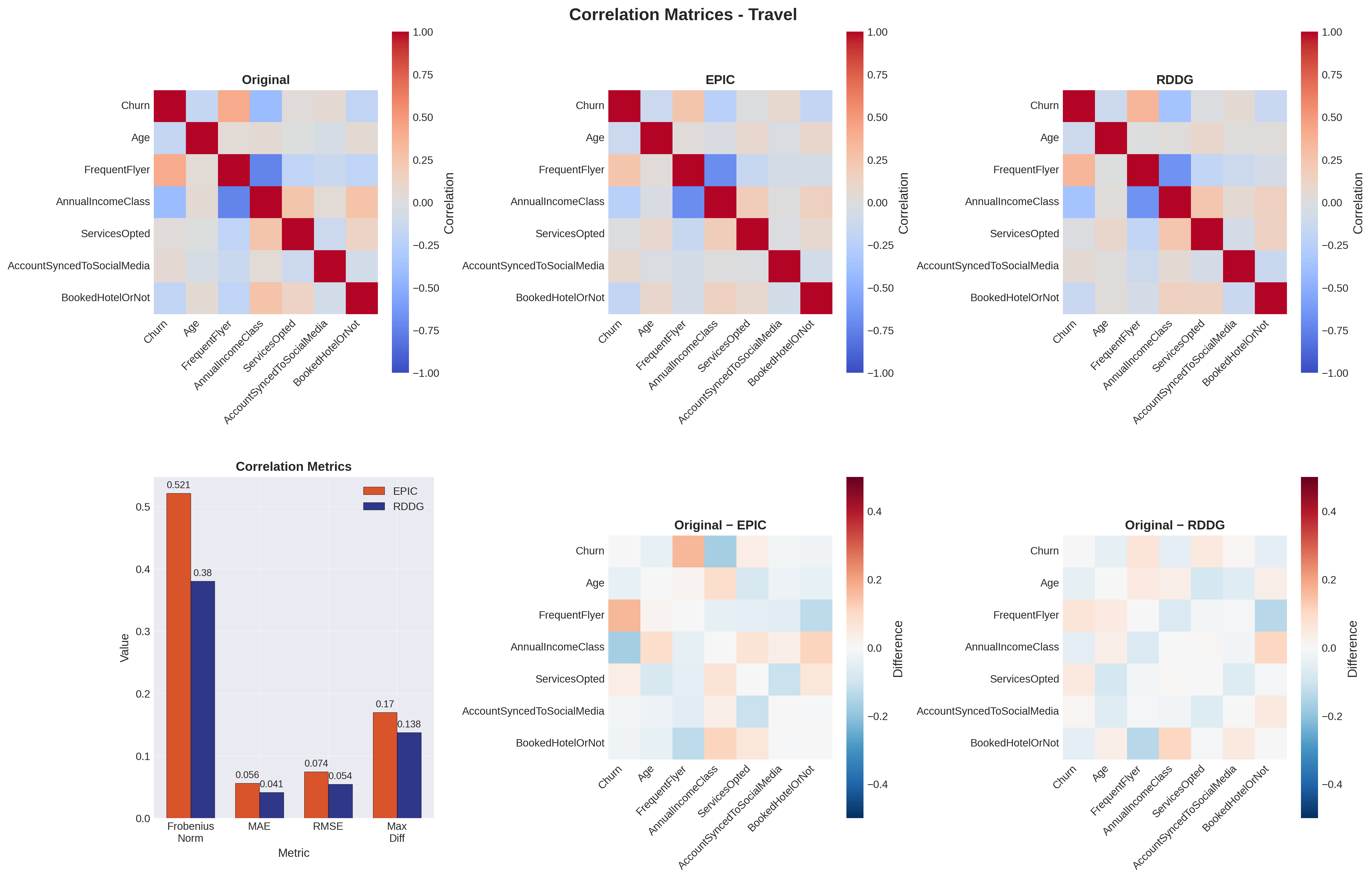}
    \caption{Correlation matrix analysis for the Travel dataset demonstrating superior correlation preservation by RDDG across all evaluation metrics.}
    \label{fig:corr_travel}
\end{figure}

\subsubsection{Inter-Attribute Correlation Preservation}

Preserving correlation structures is crucial for maintaining data fidelity in synthetic data. Figure~\ref{fig:corr_thyroid} presents the correlation preservation analysis results on Thyroid. The first three subfigures display the inter-attribute Pearson correlation coefficients for the original data and the two synthetic datasets from EPIC and RDDG, respectively. The last two subfigures show the correlation difference matrix between the original data and each of the two synthetic datasets, with lighter shades indicating better preservation of inter-attribute correlations. We observe that RDDG exhibits substantial superiority in preserving inter-attribute correlations. Finally, the fourth subfigure calculates the overall inter-attribute correlation preservation degree using Frobenius norm, MAE, RMSE, and Max Diff, showing remarkable improvements over RDDG: Frobenius norm of 1.27 versus EPIC's 4.29 (70.4\% improvement), MAE of 0.065 versus 0.173 (62.4\% improvement), RMSE of 0.075 versus 0.252 (70.2\% improvement), and Max Diff of 0.269 versus 1.21 (77.8\% improvement).

Similar observations also hold for the Travel and Real Estate datasets (Figures \ref{fig:corr_travel} and \ref{fig:corr_realestate}), where RDDG consistently outperforms EPIC across all correlation preservation metrics, maintaining inter-attribute structural relationships with substantially higher fidelity. These results indicate that RDDG preserves both distributional properties and inter-attribute relationships more effectively than EPIC. 

\subsection{Detailed Ablation Study Results} \label{Appendix-ablation}

To investigate the impact of core set selection, we compare classification performance between core sets and randomly sampled subsets on the Travel and Thyroid datasets. Note that the reference set size is set to 20 for the Thyroid dataset in this set of ablation studies. As shown in Table~\ref{tab:travel_thyroid_results}, the core set algorithm achieves substantial performance gains over random sampling. Furthermore, we demonstrate the effectiveness of our dynamic guidance (self-reinforcing feedback) mechanism, which significantly enhances overall performance, as evidenced in Table~\ref{tab:travel_thyroid_results}.

To examine how the imbalance ratio (IR) affects the performance of EPIC and RDDG, we use the UCI Segmentation dataset\footnote{https://archive.ics.uci.edu/dataset/50/image+segmentation.} and artificially vary the ratio between majority- and minority-class samples using an exponentially decaying strategy. As reported in Table~\ref{ir_result}, RDDG consistently outperforms EPIC across all metrics under varying IR levels, and its advantage over EPIC becomes more obvious as IR increases.

\begin{table*}[htbp]
\centering

\resizebox{0.65\textwidth}{!}{
\begin{tabular}{l l c c c c}
\hline
\textbf{Dataset} & \textbf{Method} & \textbf{Macro-F1} & \textbf{BAL ACC} & \textbf{Sensitivity} & \textbf{Specificity} \\
\hline
\multirow{3}{*}{Travel} 
  & RDDG        & 68.51$\pm$2.11 & 79.52$\pm$5.16 & 79.16$\pm$0.11 & 83.55$\pm$2.64 \\
  &  RDDG w/o Feedback Mechanism & 66.91$\pm$2.10 & 77.21$\pm$1.21 & 78.12$\pm$2.35 & 82.23$\pm$1.34 \\
  & RDDG w/o CoreSet & 67.55$\pm$1.32 & 78.13$\pm$2.47 & 79.00$\pm$1.32 & 82.96$\pm$2.76 \\
\addlinespace
\hline
\addlinespace
\multirow{3}{*}{Thyroid} 
  & RDDG        & 97.30$\pm$1.31 & 97.37$\pm$0.44 & 94.74$\pm$0.00 & 100.00$\pm$0.88 \\
  & RDDG w/o Feedback Mechanism & 94.74$\pm$0.00 & 96.51$\pm$0.00 & 94.74$\pm$0.00 & 98.28$\pm$0.00 \\
  & RDDG w/o CoreSet & 96.02$\pm$1.31 & 96.94$\pm$0.44 & 94.74$\pm$0.00 & 99.14$\pm$0.88 \\

\hline
\end{tabular}}
\caption{Ablation study on Travel and Thyroid datasets.}
\label{tab:travel_thyroid_results}
\end{table*}

\begin{table*}[htbp]
\centering
\resizebox{0.95\textwidth}{!}{
\begin{tabular}{ccccccccccccc}
\hline
\textbf{Dataset} & \textbf{IR} & \textbf{Method} & \textbf{Macro-F1} & \textbf{BAL ACC} & \textbf{Sensitivity} & \textbf{Specificity} & \textbf{Token} & \textbf{Expenses} & \textbf{Prompt 1} & \textbf{Prompt 2} & \textbf{Prompt 3} & \textbf{All} \\ \hline

\multirow{8}{*}{Segmentation} 
& \multirow{2}{*}{2} & EPIC  & 95.65$\pm$1.55 & 95.64$\pm$1.56 & 95.64$\pm$1.56 & 98.61$\pm$0.26 & 360K & \$0.27 & - & - & - & 312.4s \\ 
& & RDDG & 96.32$\pm$1.95 & 96.29$\pm$1.96 & 96.29$\pm$1.96 & 99.38$\pm$0.33 & 317K & \$0.25 & 7.4s & 2.3s & 231.5s & 279.6s \\ \cline{2-13}

& \multirow{2}{*}{5} & EPIC  & 95.80$\pm$1.41 & 95.76$\pm$1.43 & 95.76$\pm$1.43 & 99.29$\pm$0.24 & 400K & \$0.32 & - & - & - & 321.2s \\ 
& & RDDG & 96.01$\pm$1.29 & 95.98$\pm$1.31 & 95.98$\pm$1.31 & 99.33$\pm$0.22 & 320K & \$0.26 & 4.8s & 4.8s & 263.9s & 295.8s \\ \cline{2-13}

& \multirow{2}{*}{10} & EPIC  & 94.43$\pm$1.29 & 94.40$\pm$1.29 & 94.40$\pm$1.29 & 99.07$\pm$0.22 & 330K & \$0.26 & - & - & - & 269.3s \\ 
& & RDDG & 95.62$\pm$0.90 & 95.55$\pm$0.93 & 95.55$\pm$0.93 & 99.26$\pm$0.16 & 360K & \$0.29 & 4.9s & 2.7s & 341.0s & 358.4s \\ \cline{2-13}

& \multirow{2}{*}{20} & EPIC  & 90.53$\pm$1.93 & 91.57$\pm$1.88 & 91.57$\pm$1.88 & 97.93$\pm$0.31 & 370K & \$0.20 & - & - & - & 284.7s \\ 
& & RDDG & 91.98$\pm$1.20 & 92.02$\pm$1.20 & 92.02$\pm$1.20 & 98.67$\pm$0.20 & 1343K & \$0.81 & 4.9s & 2.1s & 719.9s & 746.8s \\ \hline

\end{tabular}}
\caption{Ablation study results on the Segmentation dataset under different imbalance ratios (IR).} \label{ir_result}
\end{table*}

\begin{table}[htbp]
\centering
\resizebox{0.85\textwidth}{!}{
\begin{tabular}{c|c|cccc|cccccc}  
\hline
\textbf{Dataset} & \textbf{Ref. size} & \textbf{Macro-F1} & \textbf{BAL ACC} & \textbf{Sensitivity} & \textbf{Specificity} & \textbf{Token} & \textbf{Expenses} & \textbf{Prompt 1} & \textbf{Prompt 2} & \textbf{Prompt 3} & \textbf{All} \\
\hline
\multirow{4}{*}{Travel} 
& 5  & 64.23$\pm$2.97 & 75.23$\pm$5.89 & 72.23$\pm$1.98 & \textbf{84.65$\pm$2.34} & 110K & \$0.09 & 5.1s & 2.2s & 43.0s & 61.0s \\
& 10 & \underline{68.51$\pm$2.11} & \underline{79.52$\pm$5.16} & \textbf{79.16$\pm$0.11} & \underline{83.55$\pm$2.64} & 160K & \$0.10 & 5.7s & 2.3s & 30.4s & 51.2s \\
& 20 & 66.23$\pm$1.23 & 79.23$\pm$5.23 & \underline{78.99$\pm$1.21} & 83.25$\pm$1.43 & 130K & \$0.11 & 6.9s & 3.1s & 14.2s & 34.1s \\
& 30 & \textbf{68.63$\pm$2.12} & \textbf{79.67$\pm$4.68} & 78.23$\pm$1.23 & 82.67$\pm$2.56 & 190K & \$0.13 & 6.1s & 2.1s & 23.4s & 41.7s \\

\hline

\multirow{4}{*}{Thyroid} 
& 5  & 95.38$\pm$1.14 & 96.72$\pm$0.38 & 94.72$\pm$0.00 & 98.71$\pm$0.77 & 358K & \$0.24 & 5.9s & 2.5s & 150.5s & 172.7s \\
& 10 & 95.41$\pm$2.13 & \underline{96.73$\pm$0.73} & \underline{94.73$\pm$0.00} & 98.71$\pm$1.47 & 353K & \$0.24 & 5.9s & 2.3s & 107.2s & 127.3s \\
& 20 & \textbf{97.30$\pm$1.31} & \textbf{97.30$\pm$0.44} & \textbf{94.74$\pm$0.00} & \textbf{100.00$\pm$0.88} & 352K & \$0.24 & 8.4s & 2.3s & 99.9s & 126.2s \\
& 30 & \underline{96.58$\pm$1.27} & 96.71$\pm$1.17 & 93.42$\pm$2.34 & \underline{99.95$\pm$0.00} & 440K & \$0.29 & 6.2s & 2.6s & 82.1s & 100.9s \\
\hline
\end{tabular}}
\caption{Ablation study on the influence of reference set size on the overall performance.} \label{bs_influ}
\end{table}

In Table \ref{bs_influ}, we also study the influence of the reference set size (\textit{Ref. size}) on the overall performance of RDDG. As with EPIC, the default size is fixed to 30 (15 samples per class) for all four real datasets. We show that it can be manually tuned for different datasets to further improve the overall performance.

\subsection{Datasets, Evaluation Metrics and Implementation Details} \label{Appendix-data}

\noindent \textbf{Datasets}. Tables \ref{dataset-des} and \ref{dataset-meta} present the statistics and descriptions of the datasets and their attributes. For the four synthetic datasets, to study the preservation of inter-attribute correlations, we manually define correlations among some attributes before generating the corresponding values. For instance, in the \textit{Real Estate} dataset, the \textit{price} of an apartment is defined to be the multiplication of \textit{basic price per square meter}, \textit{apartment size}, and \textit{renovation level}, minus age discount. In our code repository, we also provide code to generate these four datasets. 

\begin{table*}[htbp]
\centering
\resizebox{\textwidth}{!}{
\begin{tabular}{lccccc}
\toprule
\textbf{Dataset}  & \textbf{Num. attributes} & \textbf{Num. samples} & \textbf{Num. classes} & \textbf{Class samples} & \textbf{IR}  \\
\midrule
Travel           & 6  & 894   & 2 & \{0: 678, 1: 216\} & 3.139 \\
Sick              & 27 & 3711  & 2 & \{'negative': 3480, 'sick': 231\} & 15.065\\
Heloc             & 23 & 10459 & 2 & \{'Bad': 4364, 'Good': 4003\} & 1.090 \\
Thyroid           & 16 & 383   & 2 & \{'No': 275, 'Yes': 108\} & 2.546 \\
\addlinespace
\hline
\addlinespace
Consumer Behavior  & 9  & 1000  & 2 & \{Home: 518, Food: 482\} & 1.075 \\
Health Metrics     & 9  & 1000  & 3 & \{low risk: 500, medium risk: 300, high risk: 200\} & 2.500 \\
Real Estate        & 8  & 1000  & 2 & \{no: 788, yes: 212\} & 3.717 \\
Social Network     & 9  & 1000  & 4 & \{0: 789, 1: 86, 3: 70, 2: 55\} & 14.345 \\
\bottomrule
\end{tabular}}
\caption{Overview of representative dataset statistics.} \label{dataset-des}
\end{table*}

\begin{table*}[htb!]
\centering
\resizebox{\textwidth}{!}{
\small{
\begin{tabular}{p{0.08\textwidth}p{0.92\textwidth}} 
\toprule
\textbf{Dataset}  & \textbf{Attributes} \\
\midrule
Sick & The Sick dataset contains patient features such as age, sex, and thyroid-related test indicators, along with corresponding diagnosis results (sick or negative). \\
\addlinespace
Travel & The Travel dataset contains customer features such as age, frequent flyer status, annual income class, services opted, account synchronization to social media, and hotel booking status, along with corresponding churn labels.\\
\addlinespace
Thyroid & The Thyroid dataset contains patient features such as age, gender, smoking history, thyroid function, physical examination findings, pathology, and tumor staging information, along with corresponding recurrence outcomes. \\
\addlinespace
Heloc & The Heloc dataset contains credit risk features such as trade information, external risk estimates, and payment history, along with corresponding risk performance labels (Bad or Good). \\
\addlinespace
\hline
\addlinespace
Consumer Behavior & The Consumer Behavior dataset contains customer demographic information such as age, gender, income, spending score, education level, marital status, children, and location, along with corresponding product categories (food or home). \\
\addlinespace
Health Metrics & The Health Metrics dataset contains patient health indicators such as age, gender, height, weight, heart rate, blood pressure, and cholesterol levels, along with corresponding risk levels (low, medium, or high). \\
\addlinespace
Real Estate & The Real Estate dataset contains property features such as area, location, age, renovation level, price, house type, and traffic convenience, along with corresponding school district indicators (yes or no). \\
\addlinespace
Social Network & The Social Network dataset contains user social media features such as age, country, daily posts, following counts, followers counts, average likes, average likes from following, and account age exponent, along with corresponding influence levels (ranging from 0 to 3). \\
\bottomrule
\end{tabular}}
}
\caption{Attributes in different datasets.} \label{dataset-meta}
\end{table*}

\begin{table*}[htb!]
\centering
\resizebox{\textwidth}{!}{
\begin{tabular}{lcccccccc}
\toprule
\textbf{Dataset} & \textbf{Input Dim} & \textbf{Output Dim} & \textbf{MLP Architecture} & \textbf{Batch Size} & \textbf{LR} & \textbf{Optimizer} & \textbf{Epochs} \\
\midrule
Travel  & 6  & 2 & Input $\to$ Attn $\to$ Block(64) $\to$ Block(32) $\to$ Block(16) $\to$ Linear(8) $\to$ Linear(2) & 64 & 0.001 & Adam ($\beta_1=0.5,\ \beta_2=0.9$) & 100 \\
Heloc  & 22  & 2 & Input $\to$ Attn $\to$ Block(64) $\to$ Block(32) $\to$ Block(16) $\to$ Linear(8) $\to$ Linear(2) & 64 & 0.001 & Adam & 100 \\
Sick      & 27  & 2  & Input $\to$ Attn $\to$ Block(64) $\to$ Block(32) $\to$ Block(16) $\to$ Linear(8) $\to$ Linear(2) & 64 & 0.001 & Adam & 100 \\
Thyroid  & 16  & 2 & Input $\to$ Attn $\to$ Block(64) $\to$ Block(32) $\to$ Block(16) $\to$ Linear(8) $\to$ Linear(2) & 64 & 0.001 & Adam & 100 \\
\addlinespace
\hline
\addlinespace
Consumer Behavior  & 8  & 2 & Input $\to$ Attn $\to$ Block(64) $\to$ Block(32) $\to$ Block(16) $\to$ Linear(8) $\to$ Linear(2) & 64 & 0.001 & Adam & 100 \\
Health Metrics  & 8  & 3 & Input $\to$ Attn $\to$ Block(64) $\to$ Block(32) $\to$ Block(16) $\to$ Linear(8) $\to$ Linear(3) & 64 & 0.001 & Adam & 100 \\
Real Estate  & 7  & 2 & Input $\to$ Attn $\to$ Block(64) $\to$ Block(32) $\to$ Block(16) $\to$ Linear(8) $\to$ Linear(2) & 64 & 0.001 & Adam & 100 \\
Social Network  & 8  & 4 & Input $\to$ Attn $\to$ Block(64) $\to$ Block(32) $\to$ Block(16) $\to$ Linear(8) $\to$ Linear(4) & 64 & 0.001 & Adam & 100 \\
\bottomrule
\end{tabular}}
\caption{Implementation details of the Core Set algorithm.} \label{coresetconfig}
\end{table*}

\noindent \textbf{Detailed Explanations on the Evaluation Metrics}. In this work, to be consistent with EPIC, we also adopt the weighted Macro-F1 Score, Balanced Accuracy (BAL ACC), Sensitivity, and Specificity as the main evaluation metrics, which measure the overall imbalanced classification performance, the average of per-class recalls across all classes, and the corresponding recalls for the minority and majority classes, respectively. 

In classification metrics, sensitivity is essentially the recall of the positive (minority) class, measuring the model's accuracy on samples from the minority class. Conversely, specificity is the recall for the negative (majority) class, measuring the model's accuracy on majority-class samples. Balanced accuracy is defined as the average of per-class recall across all classes (majority and minority).
 
The F1 score is the harmonic mean of precision and recall. In multi-class scenarios, the weighted Macro-F1 Score is computed by: (1) calculating per-class F1 scores using a one-vs-rest approach (treating each class as positive and all others as negative), then (2) taking the weighted average of these F1 scores, where weights are determined by each class's support (number of samples).

\textbf{Implementation of Core Set}. In Table \ref{coresetconfig}, we give the implementation details and training configurations of the Core Set algorithm. We set K to be \textit{100} for core set selection when choosing the Top-K most representative samples for each class.

\subsection{Prompt Design} \label{Appendix-prompt}

\begin{table*}[htb!]
\centering
\resizebox{\linewidth}{!}{
\small{
\begin{tabular}{@{}p{0.1\textwidth}p{0.1\textwidth}p{0.2\textwidth}p{0.3\textwidth}p{0.3\textwidth}@{}}
\toprule
\textbf{Stage} & \textbf{Prompt Name} & \textbf{Main Purpose} & \textbf{Input Information} & \textbf{Output Information} \\
\midrule
\textbf{Initialization} & 
MetaData Description & 
Inject domain knowledge and establish category label semantics & 
\textit{Domain definitions:} Class labels (hypothyroidism), patient demographics (age, sex), laboratory values (TSH, T3, TT4, T4U, FTI) & 
Contextual grounding for LLMs (no direct output) \\
\addlinespace
\textbf{Prompt 1} & 
Relationship Analysis & 
Guide model to explore variable interactions and correlations & 
\textit{Analysis directive:} core set and the MetaData description as domain knowledge & 
Statistical/semantic associations (e.g., ``High TSH correlates with hypothyroidism'') \\
\addlinespace
\textbf{Prompt 2} & 
Constraint Derivation & 
Extract rules and constraints from prior analysis & 
\textit{Prior analysis results} + directive to define generation rules and constraints & 
Qualitative and quantitative rules for structured generation \\
\addlinespace
\textbf{Prompt 3} & 
Data Generation & 
Generate structured synthetic data using derived constraints & 
\textit{Derived constraints} + class balance requirements & 
Structured synthetic samples with balanced class distribution \\
\bottomrule
\end{tabular}
}
}
\caption{Overview of the three-stage prompt design for synthetic data generation. Each stage progressively builds upon the previous one to establish domain knowledge, analyze relationships, and generate structured synthetic data.}
\label{tab:prompt_design}
\end{table*}

\begin{table*}[htbp]
\centering
\small{
\begin{tabular}{@{}p{0.25\textwidth} p{0.75\textwidth}@{}}
\toprule
\textbf{Prompt} & \textbf{Output} \\
\midrule
Prompt 1 & 
Please analyze the relationships between these features and the churn (Churn) class. Identify any significant correlations or patterns that could help predict customer churn. Identify potential interactions among these features that may provide insights into customer behavior and churn likelihood.  \\
\midrule
Output1(analysis\_results) & 
From the given data, it appears that there are some potential relationships between the features and customer churn. 1. Age: In some cases, younger customers are more likely to churn compared to older customers... \\
\midrule

Prompt 2 &
\{analysis\_results\}.
Based on the above background data and the relationships among the data, rules and constraints for data generation are established. \\
\midrule
Output2(constraints) &
Rule 1: Customers who are in their 20s and opt for more services are more likely to churn compared to older customers who opt for fewer services. Rule 2: Frequent flyers are less likely to churn compared to customers who do not frequently fly. Rule 3:... \\

\midrule
Prompt 3\_1 &
\{constraints\}, ensure the class generation is balanced. \\
\midrule
Prompt 3\_2 & 
You are generating tabular data. Here is the quality evaluation report: 1. Mean and Standard Deviation Differences:   - Age: Mean diff = 0.03, Std dev diff = 0.30.... \\
\midrule
output of Prompt 3(Prompt 3\_1 + Prompt 3\_2) & 
Churn Age FrequentFlyer AnnualIncomeClass ServicesOpted AccountSyncedToSocialMedia BookedHotelOrNot
\\
 & A. \\
 & IHU,30.0,YMP,CL2,4.0,NXU,U0X
 \\
 &
IHU,31.0,YBW,OI8,2.0,NXU,EUA \\
 &
B. \\
 & 
HRL,29.0,YMP,T6L,2.0,NXU,EUA \\
 & 
HRL,31.0,YBW,CL2,4.0,NXU,U0X
\\

\bottomrule
\end{tabular}
}
\caption{Prompt Examples on the Travel Dataset.} \label{fullprompt1}
\end{table*}

\begin{table*}[htbp]
\centering
\small{
\begin{tabular}{@{}p{0.25\textwidth} p{0.75\textwidth}@{}}
\toprule
\textbf{Prompt} & \textbf{Output} \\
\midrule

Prompt 1 & 
Please analyze the relationships between these features and the recurrence of thyroid cancer (Recurred), identifying any significant correlations or patterns that could help predict cancer recurrence and potential interactions among features.  \\
\midrule
Output1(analysis\_results) & 
From the given data, we can identify potential patterns and correlations that could help predict cancer recurrence: 1. Age: Older age may be correlated with a higher likelihood of cancer recurrence. ... \\
\midrule

Prompt 2 &
\{analysis\_results\}.
Based on the above background data and the relationships among the data, rules and constraints for data generation are established. \\
\midrule
Output2(constraints) &

Rules and constraints for data generation based on the relationships between the features and the recurrence of thyroid cancer could include: 1. Age must be recorded accurately and consistently, as older age may be correlated with a higher likelihood of cancer recurrence... \\

\midrule
Prompt 3\_1 &
\{constraints\}, Ensure the class generation is balanced. \\
\midrule

Prompt 3\_2 & 
You are generating tabular data. Here is the quality evaluation report: 1. Mean and Standard Deviation Differences:   - Age: Mean diff = 1.09, Std dev diff = 2.80.... \\
\midrule

Output of Prompt 3(Prompt 3\_1 + Prompt 3\_2) & 
Recurred, Age, Gender, Smoking, Hx Smoking, Hx Radiotherapy, Thyroid Function, Physical Examination, Adenopathy, Pathology, Focality, Risk, T, N, M, Stage, Response
\\
& A. \\
 & A8O,39,A6I,GQP,Z2Y,BFG,BMN,KMR,P1R,
 \\
  &
 VDC,IOU,EOT,B8U,OLC,QA8,WY1,I8L
 \\
 &
A8O,26,LPT,GQP,Z2Y,BFG,HLJ,KMR,P1R,
\\
 & 
VDC,UE4,EOT,B8U,T47,QA8,WY1,I8L \\
 &
B. \\
 & 
N5Q,53,LPT,W6O,Z2Y,BFG,BMN,MQ8,P1R,
\\
 & 
VDC,UE4,HGR,B8U,T47,QA8,WY1,GC4
 \\
 & 
N5Q,35,A6I,GQP,Z2Y,BFG,BMN,MQ8,P1R,
\\
 & 
VDC,IOU,EOT,B8U,OLC,QA8,WY1,LSU\\
\bottomrule
\end{tabular}
}
\caption{Prompt Examples on the Thyroid Dataset.} \label{fullprompt2}
\end{table*}

\noindent In Table \ref{tab:prompt_design}, we give a sketch of our three-stages prompting strategy:  

\noindent \textbf{Initialization}. As with EPIC, in the initial step, we construct a contextualized metadata description with background explanations to define and briefly describe key variables (e.g., TSH, T3, FTI), thereby infusing domain knowledge into the language model. Furthermore, categorical samples are annotated with alphanumeric identifiers (e.g., A, B, C, D) to enhance discriminative awareness of class labels. This prompt serves as the contextual basis.

\noindent \textbf{(i) Relationship Analysis Prompt Phase}. This stage employs carefully designed prompts to guide the LLMs in analyzing relationships between features and target variables, leveraging both knowledge and statistical analysis perspectives. By inputting small-scale core-set samples, the model identifies potential correlations and interaction terms, thereby establishing plausible inter-attribute relationships that can be used to generate subsequent data.

\noindent \textbf{(ii) Constraint Derivation Prompt Phase}. Following the initial attribute relationship analysis, we use these findings as contextual input to guide the model in refining and generating data construction rules under constraints. These rules encompass qualitative or quantitative inter-attribute relationships (e.g., ``TSH is positively correlated with T3''), thereby providing explicit constraint specifications for the generative model.

\noindent \textbf{(iii) Data Generation Prompt Phase}. In the final stage, we leverage the constraints derived in the preceding phase and iteratively feed batches of the dataset into the LLMs to synthesize more representative and balanced training data. Crucially, this phase incorporates a dynamic guidance mechanism in which each generated batch is immediately assessed for quality. The evaluation results are transformed into structured guidance prompts that direct subsequent generation iterations, ensuring continuous quality improvement and adherence to the established constraints. This process ensures both usability and diversity of the generated data, thereby enhancing performance for imbalanced classification tasks.

\noindent Finally, Tables \ref{fullprompt1} and \ref{fullprompt2} provide the concrete prompts used in RDDG on the Travel and Thyroid datasets, respectively. 

\subsection{Comparisons Between EPIC and RDDG} \label{Appendix-algocompare}

As LLM-based in-context learning approaches for tabular data synthesis, both EPIC and RDDG are model-training-free and highly efficient, yielding high-quality generated data. Both approaches adopt batch-wise synthetic data generation pipelines. Their main differences include: i) RDDG is a progressive framework with CoT steps, which break the in-context learning-based generation process into relation mining, data generation, and constraint optimization. In contrast, EPIC lacks such a learning design. ii) RDDG devises the self-reinforcing feedback mechanism that makes automatic assessment on the quality of the generated data in the preceding round with respect to the reference set (a batch of real data), and such evaluation results are turned into feedback prompts and incorporated in the subsequent in-context learning-based generation step. However, EPIC lacks such a feedback mechanism. iii) We formulate the self-reinforcing feedback process as a Bayesian calibration problem and establish theoretical guarantees for our framework. Specifically, we prove the Bayes-optimal performance of our approach and show that, under certain assumptions, our feedback mechanism converges to these optimal strategies. Moreover, extensive experiments demonstrate that RDDG achieves significantly better performance than EPIC in both imbalanced classification and data fidelity.

\subsection{Pseudo Code of RDDG}
\label{sec:pseudoalgo}

\begin{algorithm}[htbp]
\caption{CreateFeedback: Self-Reinforcing Feedback Mechanism}
\label{alg:feedback}
\begin{algorithmic}[1]
\REQUIRE Quality metrics $Q_{stat}$, $Q_{corr}$, $Q_{dist}$

\STATE Initialize feedback: $\mathcal{F} = \{\}$

\IF{$Q_{stat}.\text{mean\_diff} > \tau_{mean}$}
    \STATE $\mathcal{F} \leftarrow \mathcal{F} \cup \{\text{"Adjust mean values closer to: "} + \text{target\_means}\}$
\ENDIF

\IF{$Q_{stat}.\text{std\_diff} > \tau_{std}$}
    \STATE $\mathcal{F} \leftarrow \mathcal{F} \cup \{\text{"Maintain variance similar to: "} + \text{target\_stds}\}$
\ENDIF

\IF{$Q_{corr}.\text{max\_diff} > \tau_{corr}$}
    \STATE Identify problematic attribute pairs $(a_i, a_j)$
    \STATE $\mathcal{F} \leftarrow \mathcal{F} \cup \{\text{"Strengthen correlation between "} + (a_i, a_j)\}$
\ENDIF

\IF{$Q_{dist}.\text{ks\_statistic} > \tau_{ks}$}
    \STATE Identify distribution deviations
    \STATE $\mathcal{F} \leftarrow \mathcal{F} \cup \{\text{"Align distribution patterns for: "} + \text{attributes}\}$
\ENDIF

\STATE Format feedback as structured prompt guidance
\RETURN $\mathcal{F}$

\end{algorithmic}
\end{algorithm}

\begin{algorithm}[htbp]
\caption{RDDG: Relational Data Generator with Dynamic Guidance}
\label{alg:rddg}
\begin{algorithmic}[1]
\REQUIRE Training dataset $\mathcal{D}_{train} = \{(\mathbf{x}_i, y_i)\}_{i=1}^N$, Target synthetic samples $N_{target}$, Reference set size $B$, Core set size per class $K$, LLM model $\mathcal{M}$
\ENSURE Synthetic dataset $\mathcal{D}_{syn}$

\STATE \textbf{Core Set Construction}
\STATE Initialize MLP classifier $f_{\theta}$
\STATE Partition training into phases: $\mathcal{T} = \{\mathcal{T}_{early}, \mathcal{T}_{mid}, \mathcal{T}_{late}\}$
\FOR{each training phase $t \in \mathcal{T}$}
    \FOR{each sample $(\mathbf{x}_i, y_i) \in \mathcal{D}_{train}$}
        \STATE Compute L2 error: $e_i^t = \|\mathbf{y}_{pred} - \mathbf{y}_{true}\|_2^2$
    \ENDFOR
\ENDFOR
\FOR{each sample $i$}
    \STATE Compute variance: $Var_i = \sum_{t \in \mathcal{T}} \text{Var}(e_i^t)$
\ENDFOR
\FOR{each class $c \in \mathcal{C}$}
    \STATE $\mathcal{S}_c = \text{argtop}_K(\{Var_i \mid y_i = c\})$ \COMMENT{Select top-K variance samples}
    \IF{$|\mathcal{S}_c| < K$}
        \STATE Apply replacement sampling to reach $K$ samples
    \ENDIF
\ENDFOR
\STATE $\mathcal{D}_{core} = \bigcup_{c \in \mathcal{C}} \mathcal{S}_c$ \COMMENT{Combine coresets}

\STATE
\STATE \textbf{Phase 1: Relationship Analysis}
\STATE Construct metadata prompt: $P_{meta} = \text{DescribeAttributes}(\mathcal{D}_{train})$
\STATE Initialize LLM with context: $\mathcal{M}.\text{init}(P_{meta})$
\STATE Generate relationship analysis prompt: $P_{rel} = \text{BuildRelationshipPrompt}(\mathcal{D}_{core})$
\STATE Extract relationships: $\mathcal{R} = \mathcal{M}.\text{analyze}(P_{rel}, \mathcal{D}_{core})$

\STATE
\STATE \textbf{Phase 2: Constraint Derivation}
\STATE Generate constraint prompt: $P_{const} = \text{BuildConstraintPrompt}(\mathcal{R})$
\STATE Derive constraints: $\mathcal{C}_{rules} = \mathcal{M}.\text{extract\_constraints}(P_{const}, \mathcal{R})$

\STATE
\STATE \textbf{Phase 3: Batch-wise Data Generation with Dynamic Guide}
\STATE Initialize synthetic dataset: $\mathcal{D}_{syn} = \emptyset$,  $i \leftarrow 1$
\STATE Initialize feedback: $\mathcal{F}_0 = \emptyset$
\STATE Partition original data into batches: $\mathcal{B} = \{\mathcal{B}_1, \mathcal{B}_2, ..., \mathcal{B}_m\}$ where $|\mathcal{B}_j| = B$

\WHILE{$|\mathcal{D}_{syn}| < N_{target}$}
    \STATE \textbf{Generate Batch:}
    \STATE Build generation prompt: $P_{gen} = \text{BuildGenPrompt}(\mathcal{B}_i, \mathcal{C}_{rules}, \mathcal{F}_{i-1})$
    \STATE Generate samples: $\mathcal{S}_i = \mathcal{M}.\text{generate}(P_{gen})$
    
    \STATE \textbf{Quality Evaluation:}
    \STATE Compute statistical consistency: $Q_{stat} = \text{CompareStats}(\mathcal{S}_i, \mathcal{B}_i)$
    \STATE Compute correlation preservation: $Q_{corr} = \text{PearsonDiff}(\mathcal{S}_i, \mathcal{B}_i)$
    \STATE Compute distribution consistency: $Q_{dist} = \text{KSTest}(\mathcal{S}_i, \mathcal{B}_i)$
    
    \STATE \textbf{Self-reinforcing Feedback Update:}
    \STATE Generate feedback: $\mathcal{F}_i = \text{CreateFeedback}(Q_{stat}, Q_{corr}, Q_{dist})$
    
    \STATE \textbf{Update Dataset:}
    \STATE $\mathcal{D}_{syn} \leftarrow \mathcal{D}_{syn} \cup \mathcal{S}_i$
    \STATE $i \leftarrow (i \mod m) + 1$ \COMMENT{Cycle through batches}
\ENDWHILE

\RETURN $\mathcal{D}_{syn}$
\end{algorithmic}
\normalsize
\end{algorithm}

\end{document}